\documentclass{article}

\usepackage{amsmath,amsthm,amsfonts,amssymb}
\usepackage{fullpage}
\usepackage{graphicx}

\usepackage{tikz}

\newtheorem{theorem}{Theorem}
\newtheorem{lemma}{Lemma}
\newtheorem{proposition}{Proposition}
\newtheorem{corollary}{Corollary}
\newtheorem{definition}{Definition}

\newtheorem{remark}{Remark}

\newtheorem{condition}{Condition}

\usepackage{color}
\usepackage{subfigure}
\usepackage{algorithm}
\usepackage{algorithmic}

\usepackage{xspace,prettyref}

\newcommand{\reals}{{\mathbb{R}}}
\newcommand{\integers}{{\mathbb{Z}}}
\newcommand{\naturals}{{\mathbb{N}}}

\newcommand{\diff}{{\rm d}}

\newcommand{\expect}[1]{\mathbb{E}\left[ #1 \right]}

\newcommand{\prob}[1]{ \mathbb{P}\left\{ #1 \right\} }

\newcommand{\Bern}{{\rm Bern}}

\newcommand{\eg}{e.g.\xspace}
\newcommand{\ie}{i.e.\xspace}

\newrefformat{eq}{(\ref{#1})}
\newrefformat{chap}{Chapter~\ref{#1}}
\newrefformat{sec}{Section~\ref{#1}}
\newrefformat{alg}{Algorithm~\ref{#1}}
\newrefformat{fig}{Fig.~\ref{#1}}
\newrefformat{tab}{Table~\ref{#1}}
\newrefformat{rmk}{Remark~\ref{#1}}
\newrefformat{clm}{Claim~\ref{#1}}
\newrefformat{def}{Definition~\ref{#1}}
\newrefformat{cor}{Corollary~\ref{#1}}
\newrefformat{lmm}{Lemma~\ref{#1}}
\newrefformat{prop}{Proposition~\ref{#1}}
\newrefformat{app}{Appendix~\ref{#1}}
\newrefformat{hyp}{Hypothesis~\ref{#1}}
\newrefformat{thm}{Theorem~\ref{#1}}
\newrefformat{ass}{Assumption~\ref{#1}}
\newrefformat{conj}{Conjecture~\ref{#1}}
\newrefformat{cond}{Condition~\ref{#1}}

\newcommand{\iprod}[2]{\left \langle #1, #2 \right\rangle}

\newcommand{\indc}[1]{{\mathbf{1}_{\left\{{#1}\right\}}}}

\newcommand{\calE}{{\mathcal{E}}}

\newcommand{\calH}{{\mathcal{H}}}

\newcommand{\calL}{{\mathcal{L}}}
\newcommand{\calM}{{\mathcal{M}}}

\newcommand{\calP}{{\mathcal{P}}}

\newcommand{\calS}{{\mathcal{S}}}
\newcommand{\calT}{{\mathcal{T}}}

\newcommand{\calX}{{\mathcal{X}}}

\DeclareMathAlphabet{\varmathbb}{U}{bbold}{m}{n}

\newcommand{\vct}[1]{#1}
\newcommand{\mtx}[1]{#1}
\renewcommand{\hat}{\widehat}
\renewcommand{\tilde}{\widetilde}

\usepackage{bbm}

\newcommand{\MSE}{\textup{MSE}}

\newcommand{\ER}{Erd\H{o}s-R\'{e}nyi\xspace}

\newcommand{\Hold}{H\"{o}lder\xspace}

\begin{document}

\title{Rates of Convergence of Spectral Methods for Graphon Estimation}

\author{Jiaming Xu\thanks{
Jiaming Xu is with Krannert School of Management, Purdue University, West Lafayette, IN 47907.
\texttt{xu972@purdue.edu}.}
}

\maketitle
\begin{abstract}
This paper studies the problem
of estimating the grahpon model -- the underlying generating mechanism 
of a network.
Graphon estimation arises in many applications such
as  predicting missing links in networks and 
learning user preferences in recommender systems.
The graphon model deals with a random graph of 
$n$ vertices such that each pair of two vertices 
$i$ and $j$ are connected independently with probability
$\rho \times f(x_i,x_j)$, where $x_i$ is the unknown $d$-dimensional 
label of vertex $i$, $f$ is an unknown symmetric function, and $\rho$
is a scaling parameter characterizing the graph sparsity.
Recent studies have identified the minimax error rate
of estimating the graphon from a single realization
of the random graph. However, there exists a wide gap between
the known error rates of computationally efficient estimation procedures
and the minimax optimal error rate.

Here we analyze a spectral method, namely 
universal singular value
thresholding (USVT) algorithm,
in the relatively sparse regime with the average vertex degree $n\rho=\Omega(\log n)$. 
When $f$ belongs to H\"{o}lder or Sobolev space with smoothness 
index $\alpha$, we show the error rate of USVT is at most $(n\rho)^{ -2 \alpha / (2\alpha+d)}$,
approaching the  minimax optimal error rate $\log (n\rho)/(n\rho)$ for $d=1$ as $\alpha$ increases.  Furthermore, when $f$ is analytic, we show the error rate of USVT is 
at most $\log^d (n\rho)/(n\rho)$.  In the special case of stochastic block model
with $k$ blocks, the error rate of USVT is at most $k/(n\rho)$, which
is larger than the  minimax optimal error rate by at most a multiplicative factor $k/\log k$.
This coincides with the computational gap observed for community detection. A key step of our analysis is 
to derive the eigenvalue decaying rate of the edge probability matrix using piecewise polynomial approximations of the graphon function $f$. 
\end{abstract}

\section{Introduction}

Many modern systems and datasets can be represented as networks with vertices denoting the objects and edges (possibly weighted or labelled) encoding their interactions.  
Examples include online social networks such as Facebook friendship network, biological networks such as protein-protein interaction networks, and recommender systems such as movie rating datasets. A key task in network analysis is to estimate the underlying network generating mechanism, \ie, how the edges are formed in a network. It is useful for many
important applications such as studying network evolution over time~\cite{pensky2016dynamic}, predicting missing links in networks~\cite{miller2009nonparametric,Airoldi13,gao2016optimal}, learning hidden user prefererences in recommender systems~\cite{song2016blind}, and correcting errors in crowd-sourcing systems~\cite{lee2017unifying}.
In practice, we usually only observe a very small fraction of edge connections in these networks, which obscures the underlying network generating mechanism. 
For example, around $80\%$ of the molecular interactions in
cells of Yeast~\cite{yu2008high} are unknown. In Netflix movie dataset, about $99\%$ of movie ratings are missing and the observed ratings are noisy.

In this paper, we are interested in understanding when and how  the underlying network generating mechanism can be efficiently inferred from a single snapshot of a network.
We assume the
observed network is generated according to the graphon model~\cite{lovasz2006limits}.
Graphon is a powerful network model that plays a central
role in the study of large networks. It was originally developed 
as a limit of a sequence of graphs with growing sizes~\cite{lovasz2012large},
and has been applied to various network analysis problems ranging from testing graph properties
to counting homomorphisms to charactering distances between two graphs~\cite{lovasz2012large,borgs2008convergent,borgs2012convergent} to
detecting communities~\cite{bickel2009nonparametric}.
Concretely, given 
$n$ vertices, the edges are generated
independently, connecting each pair of two distinct
vertices $i$ and $j$ with a probability
\begin{align}
M_{ij} = f (x_i, x_j),  \label{eq:edge_prob}
\end{align}
where $x_i \in \calX$ is the latent feature vector of vertex $i$ that
captures various characteristics of vertex $i$; $f: \calX \times \calX \to [0,1]$ is a symmetric function
called graphon.  We assume no self loop and 
set $M_{ii}=0$ for $1 \le i \le n$.
We further assume the feature vectors $x_i$'s are  drawn 
i.i.d.\ from the measurable space $\calX$ according to 
a probability distribution $\mu$. 
Graphon model encompasses many 
existing network models as special cases. Setting $f$ to be
a constant $p$, it gives rise to \ER random graphs~\cite{Erdos59},
where each edge is formed independently with probability
$p$. In the case where $\calX$ is 
a discrete set of $k$ elements, the model specializes to the 
stochastic block model with $k$ blocks~\cite{Holland83},
where each vertex belongs to a community, 
and the edge probability between $i$ and $j$ 
depends only on which communities they are in. If
$\calX$ is a Euclidean space of dimension $d$ and 
$f(x_i,x_j)$ is a function of the Euclidean distance $\|x_i -x_j\|$, 
then the grahon model reduces to 
the latent space model~\cite{Handcock02latentspace,Handcock07}.

To further model the partial observation of the networks,
we assume every edge is observed independently with probability $\rho \in [0,1]$,
where $\rho$ may converge to $0$ as $n \to \infty$. 
Let $\mtx{A}$ denote the adjacency matrix of the resulting observed graph
with $A_{ii}=0$ by convention. 
Then conditional on $\vct{x}=(x_1, \ldots, x_n)$, 
 for $1 \le i<j \le n$, $ A_{ij}=A_{ji} $ are independently distributed as
$ \Bern \left( \rho M_{ij} \right)$.   
The problem of interest is 
to estimate the underlying network generating mechanism -- either the edge probability matrix $M$ or the graphon $f$ -- from a single observation of the network $A.$
It turns out that estimating $M$ and estimating $f$ are the twin problems, and the result in the former can be readily extended to the latter, as shown in~\cite[Section 3]{klopp2015oracle}. Thus in this paper we shall focus on 
estimating the edge probability matrix $M$. 
To measure the quality of an estimator $\hat{\mtx{M}}$ of $M$, we 
consider the mean-squared error: 
\begin{align}
 \MSE(\hat{\mtx{M}}) = (1/n^2) \mathbb{E}[ \| \mtx{M} - \hat{ \mtx{M} } \|_F^2], \label{eq:def_mse}
 \end{align}
 which is the expected difference between the estimated edge
 probability matrix and the true one in the normalized Frobenius norm. Furthermore, to
 investigate the fundamental estimation limits, we take the decision-theoretic
 approach and consider the minimax mean-squared error: $ \inf_{\mtx{\hat{M} }} \sup_{\mtx{M} \in \calM}  \MSE(\hat{\mtx{M}})$,
 where $\calM$ denotes a set of admissible edge probability matrices. 
The minimax estimation error depends on
the smoothness of graphon $f$, 
the structure of latent space $(\calX, \mu)$, 
and the observation probability $\rho.$

 There is a recent surge of interest in graphon estimation
 and various procedures have been proposed and analyzed~\cite{gao2015rate,klopp2015oracle,gao2016optimal,wolfe2013nonparametric,Airoldi13,yang2014nonparametric,chan2014consistent,cai2014iterative,
 zhang2015estimating,borgs2015private,klopp2017optimal}. 
A recent line of work~\cite{gao2015rate,klopp2015oracle,gao2016optimal} 
has characterized the minimax error rate 
in certain special regimes. 
In particular, for stochastic block model with $k$ blocks, it
is shown that the minimax error rate is  $\frac{k^2}{n^2 \rho} + \frac{ \log k}{n \rho}$. 
For fully observed graphons with $f$ being \Hold smooth on $\calX=[0,1]$ and $\rho=1$, 
the minimax error rate turns out be $n^{-1} \log k + n^{-2\alpha/(\alpha+1)},$
where $\alpha$ is the smoothness index of $f$. 
This result was extended by~\cite{klopp2015oracle,gao2016optimal} to sparse
regimes\footnote{The minimax result derived in~\cite{klopp2015oracle} contains 
minor errors. 
In particular, the minimax rate is claimed to be 
lower bounded by $\log n/(n\rho)$.  We disproved this claim and showed 
that it is possible to strictly improve this rate and
achieve $\log (n\rho)/(n \rho)$. See~\prettyref{sec:comparison} for details.} with $\rho \to 0$.


From a computational perspective, the problem appears to be much harder and
far less well-understood. 
In the special case where $f$ is $\alpha$-\Hold smooth on $\calX=[0,1]$,
a universal singular value thresholding (USVT) algorithm is shown in~\cite{chatterjee2015matrix} to achieve an error rate of $n^{-1/3} \rho^{-1/2}$.
However, this performance guarantee is rather weak  and 
far from the  minimax optimal rate  $ \log (n \rho) / (n\rho)$. 
A similar spectral method is shown in~\cite{Xu2014edge} to
 achieve a vanishing MSE when $n \rho \gg \log n$ but
without an explicit characterization of the rate of the convergence. 
The nearest-neighbor based
approach is analyzed in~\cite{song2016blind}  under a stringent assumption $n\rho  \gg \sqrt{n}$. A simple degree sorting algorithm
is shown to achieve an error rate of $ \left( \log (n \rho ) / ( n \rho ) \right)^{\alpha/(4\alpha+d)} $ for 
$\alpha \in (0,1]$ under the restrictive assumption that $\int_{0}^{1} f(x,y) \diff y$ is strictly monotone in $x$.

In summary, despite the recent significant effort devoted to developing fundamental limits
 and efficient algorithms for graphon estimation, an understanding of the 
statistical and computational aspects of graphon estimation is still lacking.
 In particular, there is a wide gap between the known performance bounds  of computationally efficient
procedures and the minimax optimal estimation rate. 
This raises a fundamental question:\\

\emph{Is there a polynomial-time algorithm that is guaranteed to achieve the 
minimax optimal rate?}\\

In this paper, we provide a partial answer to this question by analyzing the 
universal singular value
thresholding (USVT) algorithm proposed by Chatterjee~\cite{chatterjee2015matrix}.
The universal singular value thresholding is a simple and versatile method for
structured matrix estimation and has been applied to a variety of different
problems such as ranking~\cite{shah2016stochastically}.
It truncates the singular values of $A$ 
at a threshold  slightly above the spectral norm $\|A-\expect{A}\|$, and estimates 
$M$ by a properly rescaled $A$ after truncation. 
It is computationally efficient when $A$ is sparse. 
However, its performance guarantee established in~\cite{chatterjee2015matrix}
is rather weak: the total number of observed edges needs 
to be much larger than $n^{(2d+2)/(d+2)}$  to attain a vanishing MSE.
In contrast, our improved performance bound shows that 
the  total number of observed edges only needs to be  a constant factor larger than $n \log n$,
irrespective of the latent space dimension $d$. 

 
More formally, by assuming the average vertex degree $n\rho=\Omega(\log n)$
and $\calX$ is a compact subset in $\reals^d$, the mean-squared error rate of USVT  is shown to be  upper bounded
by $(n\rho)^{  - 2\alpha / (2\alpha+d)}$,
when $f$ belongs to either  $\alpha$-smooth \Hold function class $\calH(\alpha, L)$
or $\alpha$-smooth Sobolev space $\calS(\alpha,L)$. 
Interestingly, our convergence
rate of USVT closely resembles the typical rate $N^{-2\alpha/(2\alpha+d)}$
in the nonparametric regression problem~\cite{Tsybakov2008}, where $N$ denotes
the number of observations and $d$ is the function dimension.
When $d=1$, the convergence rate of USVT is approaching
the minimax optimal rate $\log (n\rho) / (n\rho)$ 
as $f$ becomes smoother, \ie, $\alpha$ increases. 
In fact, we show that if $f$ is analytic with infinitely many times differentiability\footnote{The minimax lower bound in \cite[Appendix A.1]{gao2015rate}  is only established for the $\alpha$-smooth \Hold function class for
any fixed $\alpha$. It is an open question whether the error rate of $\log (n\rho) / (n\rho)$ is minimax-optimal for analytic graphons.}, 
then the error rate is
upper bounded by $\log^d(n\rho)/(n\rho).$

In the special case of stochastic block model
with $k$ blocks, the error rate of USVT is shown to be  
$k/(n\rho)$,  which is larger than the optimal
 minimax rate by at most a multiplicative factor $k/\log k$. 
This factor coincides with the ratio of the Kesten-Stigum threshold
and information-theoretic threshold for community detection~\cite{banks-etal-colt,AbbeSandon16,Banks16}. Based on 
compelling but non-rigorous 
statistical physics arguments, it is believed
that no polynomial-time algorithms are able to detect
the communities between the KS-threshold and IT-threshold~\cite{moore2017computer}.
This coincidence indicates that $k/(n\rho)$ may be the optimal
estimation rate among all polynomial-time algorithms, and
the minimax optimal rate may not be attainable in polynomial-time.
During the preparation of this manuscript, we became aware of an 
earlier arXiv preprint~\cite[Proposition 4]{klopp2017optimal} which also 
derives the error
rate of $k/(n\rho)$.

Our proof incorporates three interesting ingredients. 
One is a characterization of the estimation error of USVT in terms of 
the tail of eigenvalues of $M$, and the spectral norm of the noise
perturbation $\|A-\expect{A}\|$, see e.g., \cite[Lemma 3]{shah2016stochastically}.
The second one is a high-probability upper bound on $\|A-\expect{A}\|$ 
using matrix concentration inequalities initially developed by~\cite{Feige05}.
The last but most important one is a characterization of the tail of eigenvalues of $M$
using piecewise polynomial approximations of $f$, which
were originally used to study the spectrum of integral operators defined by $f$~\cite{BirmanSolomjak67,Birman1977estimates}.
The piecewise \emph{constant}
approximations of $f$ have appeared in the previous work on graphon
estimation~\cite{chatterjee2015matrix,gao2015rate,klopp2015oracle},
and are sufficient for the purpose of deriving sharp
minimax estimation rates because the smoothness of $f$ beyond $\alpha=1$
does not improve the rates. However, piecewise
degree-$\lfloor \alpha \rfloor$ polynomial approximations are needed for showing 
USVT to 
achieve a faster converging rate as $\alpha$ increases.

\paragraph*{Notation} Given a measurable space $\calX$ endowed with measure $\mu$,
let $\calL^2(\calX, \mu)$ denote the space of functions $f: \calX \to \reals$ 
such that $\| f\|_2 = \left( \int_{\calX} |f|^2 \diff \mu \right)^{1/2} <\infty$. When
$\mu$ is the Lebesgue measure, we write $\calL^2(\calX)$ for simplicity.  
Let $\reals^d$ denote the $d$-dimensional Euclidean space.  
For a vector $x \in \reals^d $, let $\|x\|_2$ denote its $\ell_2$ norm and 
$\|x\|_\infty=\max_{1 \le i \le d} |x_i|$ denote its $\ell$-infinity norm.
For any matrix $\mtx{M}$, let $\|\mtx{M} \|$ denote its spectral norm and 
$\|\mtx{M} \|_F$ denote its Frobenius norm. Logarithms are natural and we adopt the convention $0 \log 0=0$.

For any positive integer $n$, let $[n]=\{1, \ldots, n\}$.
For any positive constant $\alpha$, let $\lfloor \alpha \rfloor$ denotes the 
largest integer strictly smaller than $\alpha$. For two real numbers $\alpha$ and $\beta$,
let $\alpha \wedge \beta=\min\{\alpha, \beta\}$ and $\alpha \vee \beta=\max\{ \alpha, \beta \}$. 
For any set $T \subset [n]$, let $|T|$ denote its cardinality and $T^c$ denote its complement. If $\kappa=(\kappa_1, \ldots, \kappa_d)$
is a multi-index with $\kappa_i \in \naturals$, then  $|\kappa| =\sum_{i=1}^d \kappa_i$,
$\kappa!=\prod_{i=1}^d \kappa_i!$, 
and $x^\kappa=\prod_{i=1}^d x_i^{\kappa_i}$ for a vector $x \in \reals^d.$ 
We use standard big $O$ notations,
e.g., for any sequences $\{a_n\}$ and $\{b_n\}$, $a_n=\Theta(b_n)$ or $a_n  \asymp b_n$
if there is an absolute constant $c>0$ such that $1/c\le a_n/ b_n \le c$.
Throughout the paper, we say an event occurs with high probability 
when it occurs with a probability tending to one as $n \to \infty$.

\section{Main results}
To describe our main results, we first recall the \emph{universal singular value thresholding} (USVT) algorithm
proposed in~\cite{chatterjee2015matrix}. Note that according to the graphon model \prettyref{eq:edge_prob}, 
the edge probability matrix $M$ may not be of low-rank. 
Nevertheless, it is possible that the singular values of $M$, or equivalently
magnitudes of eigenvalues, 
drop off fast enough and as a consequence $M$ is approximately low-rank. If this is indeed 
the case, then a natural idea to estimate $M$ is via low-rank approximations of $A$. In particular,
USVT truncates the singular values of $A$ 
at a proper threshold $\tau$, and estimates $M$ by the rescaled $A$ after truncation. 
\begin{algorithm}[htb]
\caption{Universal Singular Value Thresholding (USVT)~\cite{chatterjee2015matrix}}\label{alg:svt}
\begin{algorithmic}[1]
\STATE Input: $A \in \reals^{n\times n}$, $\rho \in [0,1]$ and a threshold $\tau>0$.
\STATE Let $A=\sum_{i=1}^n s_i u_i v_i^\top$ be its singular value decomposition 
with $ s_1  \ge s_2  \ge \cdots \ge s_n $. 
\STATE Let $S$ be the set of ``thresholded'' singular values: 
$$
S= \{ i: s_i \ge \tau \}.
$$
\STATE Let 
$$
\hat{A} = \sum_{i \in S} s_i  u_i v_i^\top
$$
and $\tilde{M}= \hat{A} / \rho.$
\STATE Output a matrix $\hat{M} \in [0,1]^{n\times n}$ 
such that $\hat{M}_{ii} =0$ for all $i \in [n],$ and for $1 \le i <j \le n$, 
$\hat{M}_{ij} = \hat{M}_{ji}$ and 
\begin{align*}
\hat{M}_{ij} = 
\begin{cases}
\tilde{M}_{ij}, & \quad \text{ if } \; \tilde{M}_{ij} \in [0,1] \\
1,  & \quad \text{ if } \;  \tilde{M}_{ij} >1 \\
0, & \quad \text{ if } \; \tilde{M}_{ij} <0.
\end{cases}
\end{align*}
\end{algorithmic}
\end{algorithm}

Note that  \prettyref{alg:svt} applies hard-thresholding to the singular values of $A$. Alternatively, 
we can use soft-thresholding~\cite{koltchinskii2011nuclear} and let $\hat{A} = \sum_{i \in S} (s_i-\tau)  u_i v_i^\top$. 
Our main results with the hard-thresholding 
also apply to the soft-thresholding. As argued in~\cite{chatterjee2015matrix}, the cut-off threshold $\tau$ is chosen to
be slightly above $\| A - \expect{A} \|$, so that noise is suppressed and signals corresponding to
large singular values of $\expect{A}$ are maintained. Since conditional on $\expect{A}$,
$A$ is a random matrix with independent entries bounded in $[0,1]$ of variance at most $\rho$, 
it is expected that  $\|A - \expect{A}\| \lesssim \sqrt{ n \rho}$ with high probability, 
in view of standard matrix concentration inequalities. 
This turns out to be true if the observed graph is not too sparse, \ie, 
there exists a positive constant $C$ such that
\begin{align}
n \rho \ge C \log n. \label{eq:sparsity_assum}
\end{align}
However, when the observed graph is sparse with $ n \rho = o(\log n)$,  
due to the existence of high-degree vertices, $\|A - \expect{A}\| \gg \sqrt{ n\rho}$ with 
high probability~\cite[Appendix A]{HajekWuXuSDP14}.

Motivated by the discussion above, we shall focus on the relatively sparse
regime where \prettyref{eq:sparsity_assum} holds, 
and set $\tau= c_0 \sqrt{n \rho}$ for a positive large constant $c_0$,
whose value depends on the constant $C$ in \prettyref{eq:sparsity_assum}.
It is known that with high probability,
\begin{align*}
\| A- \expect{A} \| \le \kappa  \sqrt{ n r },
\end{align*}
where 
\begin{align}
\kappa= 
\begin{cases}
4+o(1) & n\rho = \omega(\log n)  \\
2+o(1)  &  n \rho =\omega(\log^4 n)
 \end{cases},
\end{align}
see, \eg, ~\cite[Lemma 30]{HajekWuXu_one_sdp15}.
Hence, the constant $c_0$ can be set to be
a universal constant strictly 
larger than $4$ in the case of $n\rho \gg \log(n)$
and $2$ in the case of $n \rho \gg \log^4(n)$.
Notably, in these cases, the cut-off threshold $\tau$ is universal,
independent of the underlying graphon $f$.
Our first result provides an upper bound to the estimation error of USVT.

\begin{theorem}\label{thm:mse_svt}
Consider the relatively sparse regime where \prettyref{eq:sparsity_assum} holds. 
For all $c>0$ there exists a positive constant $\kappa$ such that  if $\tau= (1+\delta) \kappa \sqrt{n\rho}$
for a fixed constant $\delta>0$,
then conditional on $M$, with probability at least $1-n^{-c}$, 
$$
\frac{1}{n^2} \| \hat{M} - M \|_F^2 \le 
16 (1+\delta)^2
\min_{0 \le r \le n} \left(    \frac{ \kappa^2 r}{n\rho}    + \frac{1}{n^2\delta^2} \sum_{i \ge r+1} \lambda_i^2 (M)  \right).
$$
Furthermore, it follows that 
$$
\MSE(\hat{M}) \le 16 (1+\delta)^2
\min_{0 \le r \le n} \left(    \frac{ \kappa^2 r}{n\rho}    + \frac{1}{n^2\delta^2} \sum_{i \ge r+1} \expect{\lambda_i^2 (M) } \right) + n^{-c}. 
$$
\end{theorem}

\prettyref{thm:mse_svt} gives an upper bound to the estimation error of USVT in terms of the tail of eigenvalues of $M$ and the observation probability $\rho.$ The upper bound invovles minimization of a sum of two terms over integers $0 \le r \le n$: the first term $r/(n\rho)$ can be viewed as the estimation error for a rank-$r$ matrix;
the second term $n^{-2} \sum_{i \ge r+1} \lambda_i^2 (M)$ is the tail of eigenvalues of $M$
and charaterizes the approximation error of 
$M$ by the best rank-$r$ matrix. The optimal $r$ is chosen to achieve the best
trade-off between the estimation error and the approximaiton error. Moreover, a lighter tail of eigenvalues of $M$ implies a faster convergence rate of the estimation error. 
To characterize different tails of eigenvalues of $M$, we introduce the 
following definitions of polynomial and super-polynomial decays. 

\begin{definition}[Polynomial  decay]
We say the eigenvalues of $\mtx{M}$ asymptotically satisfy a polynomial decay with rate $\beta > 0$ if for all
integers $0 \le r \le n-1$, 
$$
\frac{1}{n^2} \sum_{i \ge r+1} \expect{ \lambda_i^2 (M) } \le c_0 r^{-\beta} + c_1 n^{-1},
$$
where $c_0$ and $c_1$ are two  constants independent of $n$ and $r.$ 
\end{definition}

\begin{definition}[Super-polynomial decay]
We say the eigenvalues of $\mtx{M}$ asymptotically satisfy a super-polynomial decay with rate $\alpha>0$ if 
for all  integers $0 \le r \le n-1$,  
$$
\frac{1}{n^2} \sum_{i \ge r+1} \expect{ \lambda_i^2 (M) } \le  
c_0 e^{- c_2 r^{\alpha} } + c_1 n^{-1} ,
$$
where $c_0, c_1, c_2$ are constants independent of $n$ and $r.$ 
\end{definition}

We remark that in the above two definitions, we allow 
for a residual term $c_1 n^{-1}$, which is responsible for
the contribution of diagonal entries of $M$. 
According to \prettyref{thm:mse_svt}, this residual
term only induces an additional $n^{-1}$ error in the 
upper bound to MSE and will not affect our main results. 
The following corollary readily follows from \prettyref{thm:mse_svt} 
by choosing the optimal $r$ according to the decay rates of eigenvalues of $M$.

\begin{corollary}\label{cor:main}
Consider the relatively sparse regime where \prettyref{eq:sparsity_assum} holds
and suppose the eigenvalues of 
$\mtx{M} $ satisfy a polynomial decay with rate $\beta>0$.
Then there exists a positive constant $\kappa>0$ such that
if $\tau = (1+\delta) \kappa \sqrt{n \rho}$ for a fixed constant $\delta>0$, 
$$
\MSE(\hat{M}) \le c'   (n \rho )^{  - \frac{\beta }{ \beta+1} }.
$$
If instead the eigenvalues of 
$\mtx{M} $ satisfy a super-polynomial decay with rates $\alpha>0$, 
then
$$
\MSE(\hat{M}) \le c'  \frac{ \left( \log (n\rho)  \right)^{1/\alpha } }{n\rho},
$$
where $c'$ is a positive constant independent of $n$.
\end{corollary}
\begin{proof}
The first conclusion follows from \prettyref{thm:mse_svt} by choosing  $c =1 $ and 
$r = \lfloor (n \rho)^{1/(\beta+1)} \rfloor $ and 
the second one follows by choosing $c = 1$ and 
$r = \lfloor \left( \log (n\rho)/c_2\right)^{1/\alpha} \rfloor$.
\end{proof}

Next we specialize our general results in different settings by
deriving the decay rates of eigenvalues of $M.$

\subsection{Stochastic block model}
We first present  results on the rate of convergence in the stochastic block model setting,
where $x_i \in \{ 1, 2, \ldots, k\}$ indicating which community that vertex $i$ belongs to.
In this case, $M_{ij}$ only depends on the communities of vertex $i$ and vertex $j$, and 
$M$ has rank at most $k$.

\begin{theorem}\label{thm:sbm_rate}
Assume \prettyref{eq:sparsity_assum} holds under the stochastic block model with $k$ blocks,
Then there exists a positive constant $\kappa>0$ such that if 
$\tau = (1+\delta) \kappa \sqrt{n\rho}$ for some fixed constant $\delta>0$,
$$
\MSE(\hat{M}) \le c'' \left[  \frac{k}{n\rho} \wedge 1 \right].
$$
where $c''$ is a positive constant depending on $\kappa$ and $\delta.$
\end{theorem}
\begin{proof}
Under the stochastic block model, $M$ is of rank at most $k$.
Thus $\lambda_i (M) =0$ for all $i \ge k+1$. Moreover, since $M_{ij} \in [0,1]$,
it follows that $\sum_{i=1}^k \lambda_i^2(M) = \|M\|_F^2  \le n^2$. 
Applying \prettyref{thm:mse_svt} with $r=0$ and $r=k$
yields the desired result.
\end{proof}

\prettyref{thm:sbm_rate} shows that the convergence rate of MSE
of USVT is at most $  \frac{k}{n\rho} \wedge 1$, while the previous result 
in~\cite{chatterjee2015matrix} establishes that the convergence rate
is at most $\sqrt{k/n}$ for $\rho=1$. During the preparation of this manuscript, 
we became aware of an earlier arXiv preprint~\cite[Proposition 4]{klopp2017optimal} which also 
proves the error
rate of $k/(n\rho).$

The minimax optimal rate
derived in~\cite{klopp2015oracle,gao2016optimal} is
 $\left( \frac{k^2}{n^2 \rho} + \frac{ \log k}{n \rho} \right) \wedge 1$. Hence,
the error rate of USVT  is
larger than the minimax optimal rate by at most a multiplicative factor of 
$k/\log k$, which resembles the computational gap observed
for community detection~\cite{banks-etal-colt,AbbeSandon16} and the related high-dimensional statistical
inference problems discussed in~\cite{Banks16}. In particular, 
it is shown in~\cite{banks-etal-colt,AbbeSandon16} 
that estimation better than randomly guessing is attainable efficiently by
spectral methods when above the Kesten-Stigum threshold,
while it is information-theoretically possible even strictly 
below the KS threshold by a multiplicative factor $k/\log k$. 
In between the KS threshold and information-theoretic threshold, 
non-trivial estimation is 
information-theoretically possible but believed to require exponential time. 
The same conclusion also holds for exact community recovery as shown in~\cite{ChenXu14}.
Due to this coincidence, it is tempting to believe that $ \frac{k}{n\rho} \wedge 1$ might be the optimal estimation
rate among all polynomial-time algorithms; however, we do not have a proof. 



\subsection{Smooth graphon}
Next we proceed to the smooth graphon setting. We assume $\calX=[0,1)^d$ for simplicity\footnote{If
$\calX$ is a compact set in $\reals^d$, then there exists a positive constant 
$a$ such that $\calX \subset [-a,a)^d$. Hence, the general compact set case
can be reduced to $\calX=[0,1)^d$ by a proper scaling.}.
There are various notions to characterize the smoothness of graphon.
In this paper, we focus on the following two notions, which are widely
adopted in the non-parametric regression literature~\cite{Tsybakov2008}.

Given a function $g: \calX \to \reals$ and a multi-index $\kappa$, 
let \begin{align}
\nabla_{\kappa} g(x) = \frac{ \partial^{|\kappa|} g(x) }{ (\partial x)^\kappa  } \label{eq:def_partial_derivative}
\end{align}
denote its partial derivative whenever it exists. 

\begin{definition}[\Hold class] \label{def:holder}
Let $\alpha$ and $L$ be two positive numbers. 
The \Hold class $\calH(\alpha, L)$ on
$\calX$ is defined as the set of functions 
$g: \calX \to \reals$ whose partial derivatives
satisfy 
\begin{align}
\sum_{\kappa: |\kappa| = \lfloor \alpha \rfloor }
\frac{1}{\kappa!}
\left| \nabla_{\kappa} g(x)  - \nabla_{\kappa} g(x') \right| 
\le L  \|x-x'\|_\infty^{ \alpha - \lfloor \alpha \rfloor}.
\label{eq:Holdercondition}
\end{align}
\end{definition}

Note that if $\alpha \in (0,1]$, then \prettyref{eq:Holdercondition}
is equivalent to the Lip-$\alpha$ condition:
\begin{align}
| g(x) - g(x') | \le L \| x - x' \|_\infty^\alpha. \label{eq:Lip_alpha}
\end{align}

One can also measure the smoothness with respect to the
underlying measure $\mu$. This leads to the consideration of 
Sobolev space. For ease of exposition, we assume $\mu$ is the
Lebesgue measure. The main results can be extended to more general 
Borel measures.

\begin{definition}[Sobolev space]\label{def:sobolev}
Let $\alpha$ and $L$ be two positive numbers.
The Sobolev space $S(\alpha, L) $ on $(\calX,\mu)$ is defined as
the set of functions $g: \calX \to \reals$ whose partial
derivatives\footnote{More generally, the Sobolev space is defined 
when only weak derivatives exist~\cite{leoni2009first}.} satsify 
$$
\sum_{\kappa: |\kappa| = \alpha} \int_{\calX} \| \nabla_{\kappa} g(x) \|_2^2 \; \diff x \le L^2,  \quad \text{for integral } \alpha,
$$
and
$$
\sum_{\kappa: |\kappa|= \lfloor \alpha \rfloor} \int_{\calX \times \calX}  
\frac{\| \nabla_{\kappa} g(x) - \nabla_{\kappa} g(y) \|_2^2 }{ \|x-y\|_2^{2 (\alpha-\lfloor\alpha \rfloor)+d }} \; \diff x  \diff y \le L^2,  \quad \text{for non-integral } \alpha.
$$

\end{definition}


Note that the graphon $f(x,y)$ is a bi-variate function.  
We treat it as a function of $x$ for every fixed $y$,
and introduce the following two conditions on $f$.

\begin{condition}[\Hold condition on $f$]\label{cond:Hold}
There exist two positive numbers $\alpha$ and $L$ such that  $f(\cdot, y) \in \calH(\alpha,L)$ for every $y \in \calX$.
\end{condition}

\begin{condition}[Sobolev condition on $f$]\label{cond:Sob}
There exist two positive numbers $\alpha$ and $L$ such that $f(\cdot, y) \in \calS(\alpha,L(y))$ for every $y$, 
where $L(y): \calX \to \reals$ satisfies that 
$\int_{\calX} L^2(y) \diff y \le L^2$. 
\end{condition}

 The following key result shows that the eigenvalues of $\mtx{M}$ drop off to zero in a polynomial rate
 depending on the smoothness index $\alpha$ of $f.$

\begin{proposition}\label{prop:decay_rate}
Suppose that $f$ satisfies either \prettyref{cond:Hold} or \prettyref{cond:Sob}.
Then there exists a constant $C=C(\alpha, L, d)$ only depending on  $\alpha,$ $L,$ and $d$
such that  for all integers $0 \le r \le n-1$,  
$$
\frac{1}{n^2} \sum_{i \ge r+1}  \expect{  \lambda_i^2 (\mtx{M} )}  
\le C(\alpha,L, d) \left( n^{-1} + r^{- 2\alpha /d } \right).
$$
\end{proposition}
\begin{remark}
In the special case where $f$ is \Hold smooth with $\alpha=1$,
 \prettyref{prop:decay_rate} has been proved in~\cite{chatterjee2015matrix}.
In particular,  it is shown in~\cite{chatterjee2015matrix} that $f$ can be well-approximated by a piecewise constant function. 
As a consequence, $M$ can be approximated by 
a rank-$r$ block matrix with $r^2$ blocks, and the entry-wise approximation error in
the squared Frobenius norm is  shown to be approximately $ r^{-2\alpha /d}$. 
The same idea can be readily extended to the case $\alpha \in [0,1]$. However,
piecewise constant approximations of $f$ 
no longer suffice for $\alpha>1$, because \Hold smoothness condition \prettyref{eq:Holdercondition} 
no longer implies Lip-$\alpha$ condition \prettyref{eq:Lip_alpha}. In fact \prettyref{eq:Lip_alpha}
with $\alpha>1$ will imply that $f \equiv C$ for some constant $C$. 
Instead, we show that
$f$ can be well approximated by 
piecewise polynomials of degree $\lfloor \alpha \rfloor$. 
\end{remark}


By combining \prettyref{prop:decay_rate}
with \prettyref{cor:main}, we immediately get the following result
on the convergence rate of the estimation error of USVT.

\begin{theorem}\label{thm:smooth_graphon}
Under the graphon estimation model, assume \prettyref{eq:sparsity_assum} holds,
and $f$ satisfies either \prettyref{cond:Hold} or \prettyref{cond:Sob}.
There exists a positive constant $\kappa$ such that if $\tau = (1+\delta) 
\kappa \sqrt{n\rho}$ for some fixed constant $\delta>0$,
then 
$$ 
\MSE(\hat{M})  \le    c'' (n \rho )^{  - \frac{2 \alpha }{ 2\alpha+d } },
$$
where $c''$ is a positive constant  independent of $n.$
\end{theorem}

\prettyref{thm:smooth_graphon} implies that
if $f$ is infinitely many times differentiable, then the
MSE of USVT converges to zero faster than $(n\rho)^{-1+\epsilon}$
for an arbitrarily small constant $\epsilon>0.$
In fact, we can prove a sharper result when $f$ is analytic, \ie, $f$ is infinitely
differentiable and its Taylor series expansion around any point in its domain
converges to the function in some neighborhood of the point.

\begin{theorem}\label{thm:analytic_graphon}
Under the graphon estimation model, suppose there there exists positive constants $a$ and $b$ 
such that for all multi-indices $\kappa$ and all $y \in \calX$
\begin{align}
\sup_{x \in \calX} 
\frac{\partial^{|\kappa|} f(x,y)}{ (\partial x)^{\kappa}} \le b a^{|\kappa|} \kappa!.
\label{eq:analytic_cond}
\end{align}
There exists positive constants $c_0$ and $c_1$ only depending on $a, b, d$ such that for all integers $0 \le r \le n-1,$ 
\begin{align}
\frac{1}{n^2} \sum_{i \ge r+1}   \lambda_i^2 (\mtx{M} ) 
\le c_1 \left(  n^{-1} +   \exp\left( - c_0 r^{1/d} \right) \right).
\label{eq:analytic_eigen}
\end{align}
Moreover, assume \prettyref{eq:sparsity_assum} holds. 
Then there exists positive constants $c',c''$ such that if $\tau = c'' \sqrt{n\rho}$,
$$ 
\MSE(\hat{M})  \le  c'   \frac{\log^{d} \left(n\rho \right) }{n\rho}.
$$
\end{theorem}

We remark that for a fixed $y \in \calX$, \prettyref{eq:analytic_cond} is a
sufficient and necessary condition for $f(\cdot,y)$ being analytic~\cite{komatsu1960characterization}.
Note that \prettyref{eq:analytic_eigen} implies the eigenvalues of 
$M$ has a super-polynomial decay with rate $\alpha=1/d$. 
Its proof is based on approximating $f(\cdot,y)$ using its Taylor series truncated
at degree $\ell \asymp r^{1/d}$. 
When $d=1$, the eigenvalues of $M$ decays to zero exponentially fast in $r$; such an exponentialy decay can be also  proved
via Chebyshev polynomial approximation of $f$ as shown in~\cite{little1984eigenvalues}.

\subsubsection{Comparison to minimax optimal rates}\label{sec:comparison}
In this section, we compare the rates of convergence of USVT for 
estimating \Hold smooth graphons to the minimax optimal rates when
the dimension of latent feature space $d=1$. 
In the dense regimes with $\rho=1$, the minimax rates of
estimating \Hold smooth graphons have been derived in~\cite{gao2015rate}:
\begin{align*}
\inf_{\hat{\mtx{M}}} \sup_{f \in \calH(\alpha, L) } \sup_{\mu \in \calP[0,1]} 
\MSE(\hat{\mtx{M}}) 
\asymp 
\begin{cases}
 n^{ -2\alpha / (\alpha+1) } , &  0 < \alpha <1 \\
 \frac{ \log n}{n } , & \alpha \ge 1,
\end{cases}
\end{align*}
where $\calP[0,1]$ denotes all probability distributions supported over $[0,1].$
The results have been extended by~\cite{klopp2015oracle} to sparse
regimes where $\rho \to 0$ as $n \to \infty.$ However, the minimax
result derived in~\cite{klopp2015oracle}  contains minor errors. In
particular, it is claimed that that the minimax rate is always lower 
bounded by $\frac{\log n}{n\rho}$. However, as we shown in \prettyref{thm:smooth_graphon}, when $d=1$, 
the error rate of
USVT for estimating $\alpha$-smooth graphon is at most $(n\rho)^{-2\alpha/(2\alpha+1)}$,
which strictly improves over $\log n/(n\rho)$ when $n \rho \ll (\log n)^{2\alpha+1}$.
Tracing the derivations in~\cite{klopp2015oracle}, we find that
the correct minimax optimal rate is given by 
\begin{align}
\inf_{\hat{\mtx{M}}} \sup_{f \in \calH(\alpha, L) } \sup_{\mu \in \calP[0,1]} 
\MSE(\hat{\mtx{M}}) 
\asymp 
\begin{cases}
1, & n \rho = O(1) \\
 \frac{ \log (n \rho) }{n\rho}, & \omega(1) \le \log( n \rho) \le \alpha \log n  
 +(\alpha+1) \log \log n \\
(n^2 \rho)^{-\alpha/(\alpha+1)}, & \log (n\rho ) \ge \alpha \log n  + (\alpha+1) \log \log n
\end{cases}
\label{eq:optimal_rate},
\end{align}
see \prettyref{app:optimal_rate} for the derivation.
Thus, as graphon gets smoother, \ie, $\alpha$ increases, the 
upper bound to the rate of convergence of USVT $(n\rho)^{-2\alpha/(2\alpha+1)}$ approaches 
the minimax optimal rate $\log (n\rho)/(n\rho)$.

\subsection{Connections to spectrum of integral operators}
 In this section, we state a useful result, connecting the eigenvalues of 
 $\mtx{M}$ to the spectrum of an integral operator defined in terms of $f.$ 
 This allows us
 to translate existing results on the decay rates of eigenvalues of integral operators to those of $\mtx{M}.$

Define an operator $\calT: L^2(\calX, \mu) \to L^2 (\calX, \mu)$ as
\begin{equation}\label{eq:operator}
\left(\calT g \right)(x) \triangleq \int_{\mathcal{X}} f(x,y) g (y) \mu ( \diff y), 	\quad  \forall g \in L^2 (\calX, \mu).
\end{equation}
where $f$ acts as a kernal function. Hence, $\mtx{M}$ can
be also viewed as a kernal matrix. 
We assume that the graphon $f$ is square-integrable, \ie,
$
\int_{\calX \times \calX} f^2(x,y ) \mu( \diff x) \mu( \diff y)< \infty. 
$
In this case, the operator $\calT$ is known as Hilbert-Schmidt integral operator,
which is compact. Therefore it admits a discrete spectrum with finite multiplicity of all of its non-zero eigenvalues (see e.g.\ \cite{kato66,koltchinskii98,vonluxburg-bousquet-belkin}). 
Moreover, any of its eigenfunctions is continuous on $\mathcal{X}$. 
Denote the eigenvalues of operator $\calT$ sorted in decreasing order by $|\lambda_1 (\calT) | \ge |\lambda_2 (\calT) | \ge \cdots$ and 
its corresponding eigenfunctions with unit $L^2(\calX, \mu)$ norm by $\phi_1, \phi_2, \cdots$. 
By the definition of $\lambda_k$ and $\phi_k$, we
have 
\begin{align}
\lim_{m \to \infty} 
\int_{\calX \times \calX} 
\left( f(x,y) - \sum_{k=1}^m \lambda_k (\calT)  \phi_k(x) \phi_k(y) \right)^2 
\mu( \diff x) \mu (\diff y) = 0, \label{eq:kernalapprox}
\end{align}
see, \eg, ~\cite[Chapter Five, Section 2.4]{kato66}. 

The following theorem upper bounds the tail of eigenvalues of $\mtx{M}$ in expectation
using the tail of eigenvalues of $\calT$. Previous results in~\cite{koltchinskii2000random} provide similar upper bounds 
to the $\ell_2$ distance between the ordered eigenvalues of $\mtx{M}$ and 
those of $\calT$.

\begin{theorem}\label{thm:operator}
For any integer $r \ge 0$, 
\begin{align}
\frac{1}{n^2} \sum_{k \ge r+1}  \expect{ \lambda_k^2 (\mtx{M}) } \le 
\sum_{k=r+1}^\infty \lambda_k^2(\calT) +
\frac{1}{n}  \sum_{k=1}^r \sum_{\ell=1}^r \lambda_k(\calT) \lambda_\ell(\calT) 
\expect{\phi_k^2 (x_1) \phi_\ell^2(x_1) }. \label{eq:bound_integral}
\end{align}
\end{theorem}
The second term on the right hand side of \prettyref{eq:bound_integral} is responsible for the contribution of the diagonal entries of $M$. When 
$\expect{\phi_k^2 (x_1) \phi_\ell^2(x_1) }$ is bounded  and $\sum_{k=1}^\infty \lambda_k(\calT) <\infty$,  this second term is on the order of $n^{-1}.$

It is well known that if the 
kernel function $f$ is smoother, the eigenvalues of $\calT$
drops to zero faster. There is vast literature on 
estimating the decay rates of the eigenvalues of $\calT$ in terms of
the smoothness conditions of $f$, see, \eg, ~\cite{krein1965introduction,Birman1977estimates, konig2013eigenvalue,delgado2014schatten}. 
\prettyref{thm:operator} allows us to translate those existing results on the decay rates of eigenvalues of $\calT$
to those of $M$, as illustrated by examples in \prettyref{sec:numerical}.

\section{Proofs}

\subsection{Proof of \prettyref{thm:mse_svt}}

We need two key auxiliary lemmas. 
The first one gives a deterministic upper bound to the 
estimation error $\| \hat{A} - \expect{A} \|_F$
in terms of  the spectral norm $\|A-\expect{A}\|$ and the eigenvalues of $M$. 
The second one is probabilistic, providing a high-probability upper bound
to the spectral norm $\|A-\expect{A}\|$.

\begin{lemma}\label{lmm:svt}
Given two $n \times m$ real matrices  $A$ and $B$,
suppose $ \tau \ge (1+\delta) \| A - B \|$ for some fixed constant $\delta>0$ and
let $A=\sum_{i=1}^n s_i(A) u_i v_i^\top$ denote its  singular value decomposition. 
For both
$$
\hat{A} = \sum_{i: s_i(A) >\tau} s_i(A) \; u_i v_i^\top 
\quad \text{ and } \quad
\hat{A} = \sum_{i: s_i(A) >\tau} \left(s_i (A)-\tau \right) u_i v_i^\top,
$$
we have that 
$$
\| \hat{A}  - B \|_F ^2 \le  
16 \min_{0 \le r \le n} \left(  \tau^2  r + 
\left( \frac{1+\delta}{\delta} \right)^2 \sum_{i \ge r+1}^n  s_i^2 (B) \right),
 $$
where $ s_1 (B) \ge s_2(B) \ge \cdots \ge s_n(B)$ are 
the singular values of $B$.  
\end{lemma}
\prettyref{lmm:svt} without explicit constants is  
proved in \cite[Lemma 3]{shah2016stochastically},
which improves on the previous result in~\cite[Lemma 3.5]{chatterjee2015matrix}.
\prettyref{lmm:svt} with slightly different constants
is proved in \cite[Theorem 1]{koltchinskii2011nuclear} for soft singular value thresholding
and in \cite[Theorem 2]{klopp2011rank} for hard singular value thresholding. Here we provide a 
short proof for completeness.

\begin{proof}
Define an integer $\ell$ as
$$
\ell = \sup \left\{ 1 \le i \le n: s_i(B) \ge \frac{\delta}{1+\delta} \tau \right\}
$$
and set $\ell=0$ by default if the above supreme is taken over the empty set.
We claim that $\hat{A}$ is of rank at most $\ell.$ Indeed, if $\ell=n,$ the claim
holds trivially. Otherwise, $s_{\ell+1} (B) < \delta \tau /(1+\delta).$ By Weyl's perturbation 
theorem and the assumption that $\tau \ge (1+\delta)\|A-B\|,$
$$
s_{\ell+1} (A) \le s_{\ell+1} (B) + \|A- B\| < \frac{\delta}{1+\delta} \tau + \frac{1}{1+\delta} \tau = \tau,
$$
and hence $\hat{A}$ is of rank at most $\ell$ by the definition of $\hat{A}.$
Let $B_\ell$ denote the best rank-$\ell$ approximation of $B.$ Then by triangle's inequality
$$
\| \hat{A} - B\|_F \le \| \hat{A} - B_\ell \|_F + \| B- B_\ell\|_F
$$
and thus
\begin{align*}
\| \hat{A} - B\|_F^2 & \le 2 \| \hat{A} - B_\ell \|_F^2  + 2 \| B- B_\ell\|_F^2 \\
& \le 4 \ell \| \hat{A} - B_\ell \|^2 + 2 \sum_{i \ge \ell+1} s_i^2(B),
\end{align*}
where the last inequality holds because $\hat{A}-B_\ell$ is of rank at most $2\ell.$
By triangle's inequality again and the fact that 
$\| \hat{A}- A \| \le \tau$, we have that 
$$
\| \hat{A} - B_\ell \| \le \| \hat{A} - A \| + \|A-B\| + \|B-B_\ell\|
\le \tau + \frac{1}{1+\delta} \tau + \frac{\delta}{1+\delta} \tau = 2\tau.
$$
Combining the last two displayed equaitons yields that 
$$
\| \hat{A} - B\|_F^2  \le 16 \ell \tau^2 + 2 \sum_{i \ge \ell+1} s_i^2(B)
\le 16 \left( \ell \tau^2 + \left( \frac{1+\delta}{\delta} \right)^2 \sum_{i \ge \ell+1} s_i^2(B)   \right).
$$
Finally, to complete the proof, note that by the definition of $\ell$, 
 for all $ 0 \le r \le n$, 
$$
\ell \tau^2 + \left( \frac{1+\delta}{\delta} \right)^2 \sum_{i \ge \ell+1} s_i^2(B)  
\le  \tau^2  r + 
\left( \frac{1+\delta}{\delta} \right)^2 \sum_{i \ge r+1}^n  s_i^2 (B). 
$$

\end{proof}

Lemma~\ref{lmm:concentration_spectral} initially developed by~\cite{Feige05} and extended by~\cite{tomozei2014,chatterjee2015matrix,HajekWuXuSDP14,BVH14}, gives upper bounds to the spectral norm of random symmetric matrices with bounded entries. 
\begin{lemma}\label{lmm:concentration_spectral}
 Let $A$ denote a symmetric and zero-diagonal random matrix, where the entries $\{A_{ij}: i<j\}$ are independent and $[0,1]$-valued..
 Assume that $\mathbb{E}[A_{ij}]  \le \rho$ for some $\rho>0$. If  \prettyref{eq:sparsity_assum} holds, \ie, 
 $n \rho \ge C \log n$ for a constant $C$, then for all $c>0$ there exists a constant $\kappa>0$ such that with probability at least $1-n^{-c}$, 
\begin{equation}
\| A -\mathbb{E}[A] \| \le \kappa \sqrt{n \rho}.
\end{equation}
\end{lemma}

\prettyref{thm:mse_svt} readily follows by combining the above two lemmas. 

\begin{proof}[ Proof of \prettyref{thm:mse_svt} ]
Let us first condition on $M$. 
For any given $c>0$, by~\prettyref{lmm:concentration_spectral},
there exists a constant $\kappa>0$ such that 
$\prob{\calE} \ge 1-n^{-c}$, where
$$
\calE \triangleq \left\{  \| A -\mathbb{E}[A]   \| \le \kappa \sqrt{n \rho} \right\}.
$$
Since in the theorem assumption $\tau=(1+\delta) \kappa \sqrt{n\rho}$ for a fixed constant $\delta>0$, 
it follows from \prettyref{lmm:svt} that on event $\calE$,
$$
\| \hat{A}  - \expect{A} \|_F ^2 \le 
 16 (1+\delta)^2  \min_{0 \le r \le n} \left( \kappa^2 n \rho r +  \frac{1}{\delta^2}  \sum_{i \ge r+1}^n  \lambda_i^2 \left(\expect{A} \right) \right),
$$
Recall that 
$\hat{A} = \rho \tilde{M}$ and $\expect{A} = \rho M$. Hence, on event $\calE$, 
$$
\frac{1}{n^2} \| \tilde{M}  - M \|_F ^2 
\le  16 (1+\delta)^2  \min_{0 \le r \le n} \left(    \frac{\kappa^2 r}{n\rho} 
+  \frac{1}{n^2 \delta^2 } 
\sum_{i \ge r+1}^n  \lambda_i^2 (M) \right).
$$
By the definition of $\hat{M}$ and the fact that $M_{ii}=0$ and $M_{ij} \in [0,1]$,
it follows that $\| \hat{M}  - M \|_F ^2 \le  \| \tilde{M}  - M \|_F ^2 $ and thus the first conclusion follows.

For the second conclusion on $\MSE(\hat{M})$, note that 
$| \hat{M}_{ij}  - M_{ij}| \in [0,1]$. Hence,
conditioning on $M$, 
\begin{align*}
\frac{1}{n^2} \expect{\| \hat{M}  - M \|_F ^2 }
& = \frac{1}{n^2} \expect{ \| \hat{M}  - M \|_F ^2 \indc{\calE} }
+\frac{1}{n^2} \expect{ \| \hat{M}  - M \|_F ^2 \indc{\calE^c} } \\
& \le 16 (1+\delta)^2  \min_{0 \le r \le n} \left(    \frac{\kappa^2 r}{n\rho} 
+  \frac{1}{n^2 \delta^2 } 
\sum_{i \ge r+1}^n  \lambda_i^2 (M) \right) \times \prob{\calE} 
+ \prob{\calE^c} \\
& \le 16 (1+\delta)^2  \min_{0 \le r \le n} \left(    \frac{\kappa^2 r}{n\rho} 
+  \frac{1}{n^2 \delta^2 } 
\sum_{i \ge r+1}^n  \lambda_i^2 (M) \right) + n^{-c}.
\end{align*}
Finally, taking the expectation of $M$ over both hand sides of the last
displayed equation, we get that 
\begin{align*}
\MSE(\hat{M} ) & \le 16 (1+\delta)^2 \expect{   \min_{0 \le r \le n} \left(    \frac{\kappa^2 r}{n\rho} 
+  \frac{1}{n^2 \delta^2 } 
\sum_{i \ge r+1}^n  \lambda_i^2 (M) \right) } + n^{-c}  \\
& \le 16 (1+\delta)^2   \min_{0 \le r \le n} \left(    \frac{\kappa^2 r}{n\rho} 
+  \frac{1}{n^2 \delta^2 } 
\sum_{i \ge r+1}^n \expect{  \lambda_i^2 (M) }\right) 
  +n^{-c},
\end{align*}
where the last inequality holds by Jensen's inequality 
because $\min_{0 \le r \le n} \left(    \frac{ \kappa^2  r}{n\rho} 
+  \frac{1}{n^2\delta^2} 
\sum_{i \ge r+1}^n  \lambda_i^2 (M)  \right) $ is concave in $\lambda_i^2(M).$ 

\end{proof}

\subsection{Proof of \prettyref{prop:decay_rate}}
In this section, we prove the decay rates of eigenvalues of $M$ when $f$ is a smooth graphon. 
The key idea of our proof is to approximate $f(\cdot, y)$ by 
a piecewise polynomial for every $y.$ We first introduce a
rigorous definition of piecewise polynomials.

\begin{definition}[Piecewise Polynomial]
Let $\calE$ denote a partition of the cube $[0,1)^{d}$ into a finite number (denoted by 
$|\calE|$) of cubes $\Delta$. Let $\ell$ denote a natural number. 
We say $P_{\calE,\ell}: [0,1]^d \to \reals $ is a piecewise polynomial of degree $\ell$ 
if
\begin{align}
P_{\calE,\ell} (\vct{x}) = \sum_{\Delta \in \calE} P_{\Delta,\ell} (\vct{x}) \indc{\vct{x} \in \Delta},  \label{eq:def_piecewise_poly}
\end{align}
where $P_{\Delta,\ell} (\vct{x}): [0,1]^d \to \reals $ denotes a polynomial of degree at most $\ell$. 
\end{definition}

For our proof, it suffices to consider an equal-partition of $[0,1)^d$. More precisely,
for every naturual $k$, $[0,1)$ is partitioned into $k$ half-open intervals of lengths $1/k$,
\ie, 
$
[0, 1) = \cup_{i=1}^k \left[ (i-1)/k, i/k \right).
$
It follows that $[0,1)^{d}$ can be partitioned into $k^{d}$ cubes of forms
$\otimes_{j=1}^{d}  \left[ (i_j-1)/k,   i_j/k \right)$ with $i_j \in [k]$.
Let $\calE_k$ be such a partition with $I_1, I_2, \ldots, I_{k^{d}} $ denoting all such cubes
and  $\vct{z}_1, \vct{z}_2, \ldots, \vct{z}_{k^{d}} \in \reals^{d} $ denoting the centers
of those cubes.

The following lemma shows that any \Hold function $g \in \calH(\alpha, L)$
can be approximated by a piecewise polynomial $P_{\calE_k,\lfloor \alpha \rfloor}$ 
of degree $\lfloor \alpha \rfloor.$
The construction of $P_{\calE_k,\lfloor \alpha \rfloor}$ is
based on Taylor expansions at  points
$\vct{z}_1, \ldots, \vct{z}_{k^{d}}.$

\begin{lemma}\label{lmm:holder_polynomial}
Suppose $g \in \calH(\alpha,L)$ and  
let $\ell=\lfloor \alpha \rfloor$.
For every natural $k$, there is a piecewise polynomial 
$P_{\calE_k, \ell} (\vct{x})$
satisfying 
$$
\sup_{ \vct{x} \in \calX } \left| g(\vct{x}) - P_{\calE_k, \ell} (\vct{x} ) \right|
\le L k^{ - \alpha}.
$$ 
\end{lemma}
\begin{proof}

For every $I_i$ with $1 \le i \le k^d$, 
define $P_{I_i, \ell} (\vct{x} ) $ as the 
degree-$\ell$ 
Taylor's series expansion of $g(\vct{x} )$ at point $\vct{z}_i$:
\begin{align}
P_{I_i, \ell} (\vct{x} ) 
=\sum_{\kappa: |\kappa| \le \ell} \frac{1}{\kappa ! } \left(\vct{x}-\vct{z_i} \right)^\kappa \nabla_\kappa g( \vct{z}_i ), \label{eq:Taylor_series}
\end{align}
where $\kappa=(\kappa_1,\ldots, \kappa_d)$ is a multi-index with $\kappa!=\prod_{i=1}^d \kappa_i!$, and
$\nabla_k g(z_i)$ is the partial derivative defined in~\prettyref{eq:def_partial_derivative}.
Define a degree-$\ell$ piecewise polynomial as in \prettyref{eq:def_piecewise_poly}, \ie, 
\begin{align}
P_{\calE_k, \ell} (\vct{x}) = 
\sum_{i=1}^{k^{d}}  P_{I_i, \ell} (\vct{x} ) 
\indc{\vct{x} \in I_i}. \label{eq:piece_wise_polynomial}
\end{align}
Since $f \in \calH(\alpha,L)$, it follows from Taylor's theorem that 
\begin{align*}
\sup_{ \vct{x} \in \calX } 
\left| g(\vct{x}) - P_{\calE_k, \ell} (\vct{x} ) \right|
&= \sup_{1 \le i \le k^{d}} \sup_{\vct{x} \in I_i}  
\left| g(\vct{x} ) - P_{I_i, \ell} (\vct{x} ) \right| \\
& \le  \sup_{1 \le i \le k^{d}}  \sup_{\vct{x} \in I_i} 
\|\vct{x} - \vct{z}_{i} \|_\infty^\ell
\sup_{\vct{x} \in I_i} 
\sum_{\kappa: |\kappa|=\ell} 
\frac{1}{\kappa!} \left| \nabla_\kappa g ( \vct{x} ) -\nabla_{\kappa} g(\vct{z}_{i} ) \right|   
 \\
& \le  L  \sup_{1 \le i \le k^{d}}  
 \sup_{\vct{x} \in I_i} \|\vct{x} - \vct{z}_{i} \|_\infty^{\alpha} =  L  k^{-\alpha}.
\end{align*}
\end{proof}

Next we proceed to the case where $g$ belongs to Sobolev space $\calS(\alpha,L).$
Let $\Delta$ be a cube in $\reals^d$. 
We define a polynomial $p$ of degree $\ell$
satisfying the conditions: for all multi-index $\kappa$ such that
$|\kappa| \le \ell$, 
$$
\int_{\Delta} x^{\kappa} p(x) \diff x = \int_{\Delta} x^\kappa g(x ) \diff x.
$$
It is clear that $p$ is uniquely defined. We let $
\left( \calP_{\Delta,\ell} \right) g  \triangleq p$
and hence $\calP_{\Delta,\ell}$ is a linear projection operator mapping the space 
$\calS(\alpha, L)$ onto the finite-dimensional space of polynomials of degree
$\ell$. 
We define
$$
\left( \calP_{\calE_k,\ell} \right) g (x) =
\sum_{i=1}^{k^d} \left(\calP_{I_i, \ell}  \right) g (x)  \indc{x \in I_i}. 
$$
In other words, $\left(\calP_{\calE_k,\ell}\right) g$ is the piecewise polynomial coinciding with
$\left(\calP_{I_i, \ell} \right) g$ on each cube $I_i$ for $1 \le i \le k^d.$
The following lemma proved in~\cite[Theorem 3.3, 3.4]{BirmanSolomjak67} upper bounds the approximation error of $g$ by 
$\left(\calP_{\calE_k,\ell}\right) g$ in $L^2(\calX, \mu)$ norm.
\begin{lemma}\label{lmm:sobolev_polynomial}
There exists a constant $C(\alpha, d)$ only depending on
$\alpha$ and $d$ such that 
for every $g \in S(\alpha,L)$ and every natural $k$, 
$$
\int_{\calX} \left| g(\vct{x} ) - \left(\calP_{\calE_k, \lfloor \alpha \rfloor } \right) g (\vct{x} ) \right|^2 \mu(\diff x)  
\le C(\alpha,  d) L^2  k^{ -  2\alpha} .
$$ 
\end{lemma}

With \prettyref{lmm:holder_polynomial}
and \prettyref{lmm:sobolev_polynomial},
we are ready to prove \prettyref{prop:decay_rate}, which provides upper bounds to
the  decay rates of eigenvalues of $\mtx{M}.$

\begin{proof}[Proof of \prettyref{prop:decay_rate}]
Let $C_0(\alpha, d) \triangleq \sum_{i=0}^{\lfloor \alpha \rfloor} \binom{i+d-1}{d-1}$.
Fix any natural $ 0 \le r \le n-1$. If $r \le 2^d C_0$, then by choosing $C(\alpha, L, d) \ge (2^d C_0)^{2\alpha/d}$, we have that 
$$
\frac{1}{n^2} \sum_{i \ge r+1} \lambda_i^2 (M) \le  \frac{1}{n^2} \|M\|_F^2 \le 1 
\le C(\alpha, L, d) r^{-2\alpha /d}, \quad \forall 0 \le r \le 2^d C_0. 
$$
Thus, it suffices to prove the conclusion for $r \ge 2^d C_0$. In this case,  
there exists a $k \ge 2$ such that $ k^d C_0 \le r \le (k+1)^d C_0$.

We first focus on the case where $f(\cdot,y) \in \calH(\alpha,L)$ for every $y \in \calX.$
In view of \prettyref{lmm:holder_polynomial},  
for every $y \in \calX$, there is a piecewise polynomial  
$P_{\calE_k,\lfloor \alpha \rfloor} (\vct{x}; \vct{y} )$
of degree $\lfloor \alpha \rfloor$ satisfying 
$$
\sup_{ \vct{x} \in \calX } \left| f(\vct{x}, \vct{y} ) - P_{\calE_k,\lfloor \alpha \rfloor} (\vct{x}; \vct{y} ) \right|
\le L k^{ - \alpha}.
$$ 
Define an $n\times n$ matrix $\mtx{N}$ such that 
$$
N_{ij} = P_{\calE_k, \lfloor \alpha \rfloor} (\vct{x}_i; \vct{x}_j).
$$
It follows that for all $1 \le i \neq j \le n$, 
$$
| M_{ij} - N_{ij} | = 
\left| f( \vct{x}_i, \vct{x}_j) -  P_{\calE_k, \lfloor \alpha \rfloor} (\vct{x}_i; \vct{x}_j) \right|
\le L k^{-\alpha}.
 $$
 Moreover, for all $ 1 \le i \le n$, since $M_{ii}=0$ by definition,
 we get that 
 $$
 | M_{ii} - N_{ii} |
 =| N_{ii}| =\left| P_{\calE_k, \lfloor \alpha \rfloor} (\vct{x}_i; \vct{x}_i) 
 \right| \le \left| f(x_i,x_i) \right| + L k^{-\alpha}
 \le 1+ L k^{-\alpha}.
 $$
By construction, $P_{\calE_k,\lfloor \alpha \rfloor} (\vct{x}; \vct{y} )$ is a
piecewise polynomial of degree $\lfloor \alpha \rfloor$ 
and hence it admits the decomposition: 
$$
P_{\calE_k,\lfloor \alpha \rfloor} (\vct{x}; \vct{y} )
=\sum_{\Delta \in \calE_k}  \iprod{\Phi(\vct{x}) }{ \beta_{\Delta, \vct{y} }}
\indc{\vct{x} \in \Delta},
$$
where 
$$
\Phi(x) = \left(1, x_1, \ldots, x_d, \ldots, x_1^{\lfloor \alpha \rfloor},
\ldots, x_d^{\lfloor \alpha \rfloor} \right)^\top
$$
denotes the vector consisting of all monomials $x^\kappa$ of degree $|\kappa| \le \lfloor \alpha \rfloor$;
and $\beta_{\Delta, \vct{y} }$ denotes the corresponding  coefficient vector. Therefore,
$$
N_{ij} = \sum_{\Delta \in \calE_k}  \iprod{\Phi(\vct{x_i}) }{ \beta_{\Delta, \vct{x_j} }}
\indc{\vct{x_i} \in \Delta},
$$
and thus
$$
\mtx{N} = \sum_{\Delta \in \calE} \begin{bmatrix}
\Phi^\top(\vct{x_1})  \indc{\vct{x_1} \in \Delta} \\
\vdots \\
\Phi^\top(\vct{x_n})  \indc{\vct{x_n} \in \Delta}
\end{bmatrix}
\begin{bmatrix}
\beta_{\Delta, \vct{x_1}} & \cdots & \beta_{\Delta, \vct{x_n}}
\end{bmatrix}.
$$

Since there are $C_0(\alpha,d)$
monomials of degree at most $\lfloor \alpha \rfloor$, it follows that 
$\Phi(\vct{x_i})$ and $\beta_{\Delta, \vct{x_j}} $ are of dimension
at most $C_0$.
Therefore, the rank of $\mtx{N}$ is at most $k^d C_0$. As a consequence, 
 \begin{align}
 \frac{1}{n^2} \sum_{i=r+1}^n \lambda_i^2 ( \mtx{M} ) &  \overset{(a)}{\le}  
  \frac{1}{n^2} \sum_{i=k^d C_0+1}^n \lambda_i^2 ( \mtx{M} ) \nonumber \\
  &   \overset{(b)}{\le}  
 \frac{1}{n^2} \|\mtx{M}- \mtx{N} \|_F^2 \le  \frac{2}{n}+
2L^2 k^{-2\alpha} \nonumber  \\
& \overset{(c)}{\le} 
\frac{2}{n}+ 2L^2  \left( \left(r/C_0\right)^{1/d} -1 \right)^{-2\alpha}  \nonumber \\
& \le    \frac{2}{n} + 2^{2\alpha+1}  L^2  C_0^{2\alpha/d} r^{-2 \alpha/d}, \label{eq:holder_bound_end}
\end{align}
where $(a)$ holds because $r \ge k^d C_0$; $(b)$ holds due to the rank of $N$ is at most $k^d C_0$;
$(c)$ holds because $r \le (k+1)^d C_0$; and the last inequality holds because $r \ge 2^d C_0$.

Next we move to the case where $f(\cdot, y) \in \calS\left(\alpha, L(y)\right)$ 
for every $y \in \calX$ and 
$\int_{\calX} L^2(y) \mu(\diff y) \le L^2$. 
For every $y \in \calX$, let 
$\left(\calP_{\calE_k,\lfloor \alpha \rfloor} \right)
f(\cdot, \vct{y} )$ denote the piecewise
polynomial approximation of $f(\cdot, \vct{y})
$ as given in \prettyref{lmm:sobolev_polynomial}.
Then it follows that for every $y \in \calX$, 
$$
\int_{X} \left| f(\vct{x}, \vct{y} ) - 
\left( \calP_{\calE_k, \lfloor \alpha \rfloor }\right) f(\vct{x} , \vct{y} ) \right|^2 \mu(\diff x)  \le C(\alpha, d) L^2(y) k^{-  2\alpha} .
$$
Define an $n\times n$ matrix $\mtx{N}$ such that 
$
N_{ij} = \left(\calP_{\calE_k, \lfloor \alpha \rfloor} \right) f (\vct{x}_i, \vct{x}_j).
$
It follows that for all $1 \le i \neq j \le n$, 
$$
\expect{| M_{ij} - N_{ij} |^2} = 
\expect{
\left| f( \vct{x}_i, \vct{x}_j) -  
\left(\calP_{\calE_k, \lfloor \alpha \rfloor} \right) f(\vct{x}_i, \vct{x}_j) \right|^2}
\le  C(\alpha,d) \expect{L^2 (x_j)} k^{-2\alpha} \le  C(\alpha, d) L^2 k^{-2\alpha},
 $$
 where we used the fact that $x_i$ and $x_j$ are independent. 
Moreover, for $ 1 \le i \le n,$ since $M_{ii}=0$ by definition
and $x_i$'s are identically distributed, 
we get that
\begin{align}
\expect{| M_{ii} - N_{ii} |^2 }
=\expect{\left| 
\left(\calP_{\calE_k, \lfloor \alpha \rfloor} \right) f(x, x) \right|^2}
=\sum_{j=1}^{k^d}  \expect{ 
\left| \left(\calP_{I_j, \lfloor \alpha \rfloor}  \right) f (x,x) \right|^2 \indc{x \in I_j} }. 
\label{eq:M_diagonal_bound}
\end{align}
where the last equality holds because
$
\left( \calP_{\calE_k, \lfloor \alpha \rfloor} \right) f (x,x) =
\sum_{j=1}^{k^d} \left(\calP_{I_j, \lfloor \alpha \rfloor}  \right) f (x,x)  \indc{x \in I_j}. 
$
Fix any $1 \le j \le k^d$, we next 
upper bound $
\left| \left( \calP_{I_j, \ell} \right) f(x, x)  \right|^2$ for $x \in I_j.$
Let $\Psi(x)=(\Psi_1(x), \ldots, \Psi_{C_0}(x) )$
denote the orthonormal basis of 
the subspace of $\calL^2(I_j)$ 
consisting of all monomials $x^\kappa$ of degree $|\kappa| \le \lfloor \alpha
\rfloor.$ It follows from the definition of  $\calP_{I_j, \ell}$ that 
$$
\left( \calP_{I_j, \ell} \right) f(x, y) =\langle \Psi(x), \beta(y)\rangle, 
$$
where $\beta(y)=(\beta_1(y), \ldots, \beta_{C_0}(y)$ is given by
$$
\beta_m(y) = 
\int_{I_j} \left( \calP_{I_j, \ell} \right) f(x, y) \Psi_m(x) \diff x
=\int_{I_j}  f(x, y) \Psi_m(x) \diff x, \quad \forall 1 \le m \le C_0,
$$
where the last equality follows from the definition of $\left( \calP_{I_j, \ell} \right) f(\cdot,y)$. 
Therefore,  by Cauchy-Schwartz inequality,
$$
\beta_m^2(y) \le \int_{I_j} f^2 (x,y) \diff x \int_{I_j} \Psi_m^2(x)\diff x
\le \int_{I_j} \diff x \int_{I_j} \Psi_m^2(x)\diff x
\le
k^{-d}, \quad \forall 1 \le m \le C_0, 
$$
where we used the fact that $f(x,y) \in [0,1]$
and that $\int_{I_j} \Psi_m^2(y)\diff y =1$.
Hence,
$$
\left| \left( \calP_{I_j, \ell} \right) f(x, x)
\right|^2
=  \langle \Psi(x), \beta(x)\rangle^2
\le \sum_{m=1}^{C_0} \Psi_m^2(x) \sum_{m=1}^{C_0} \beta_m^2(x)
\le C_0 k^{-d} \sum_{m=1}^{C_0} \Psi_m^2(x)
$$
and thus
$$
\expect{ 
\left| \left(\calP_{I_j, \lfloor \alpha \rfloor}  \right) f (x,x) \right|^2 \indc{x \in I_j}}
\le C_0 k^{-d} \sum_{m=1}^{C_0} \int_{I_j }  \Psi_m^2(x)  \diff x  =C_0^2 k^{-d}.
$$
In view of \prettyref{eq:M_diagonal_bound}, we get that for all $1 \le i \le n,$
$$
\expect{|M_{ii}-N_{ii}|^2} \le C_0^2.
$$
Since the rank of $\mtx{N}$ is at most $k^d C_0(\alpha, d)$, by the same argument
as for \prettyref{eq:holder_bound_end}, we have that 
 \begin{align*}
 \frac{1}{n^2} \sum_{i=r+1}^n \lambda_i^2 ( \mtx{M} ) 
  \le
 \frac{1}{n^2} \|\mtx{M}- \mtx{N} \|_F^2 \le  
 \frac{2C_0^2}{n} + 2 C(\alpha,d)  L^2 k^{-2\alpha}  
 \le \frac{2C_0^2}{n} + 2^{2\alpha+1}  C(\alpha,d) C_0^{2\alpha/d} L^2 r^{-2 \alpha/d},
\end{align*}
 which completes the proof.
\end{proof}

\subsection{Proof of \prettyref{thm:analytic_graphon}}


Fix two integers $k \ge 1$ and $\ell \ge 1$ to be specified later. 
Recall the degree-$\ell$ Taylor series expansion of $f(\cdot,y)$ defined in \prettyref{eq:Taylor_series}
and the piecewise polynomial of degree $\ell$ defined in \prettyref{eq:piece_wise_polynomial}.
Since $f(\cdot,y)$ is infinitely many times differentiable and the partial derivatives satisfy \prettyref{eq:analytic_cond}, 
it follows from Taylor's theorem that
$$
\sup_{x, y\in \calX} \left| f(x,y) - P_{\calE_k, \ell-1} (x;y) \right|
\le  k^{-\ell} L_\ell,
$$
where 
$$
L_\ell =  \sum_{\kappa: |\kappa| = \ell} \frac{1}{\kappa!} \sup_{x, y} \left| \frac{ \partial^{|\kappa|} f(x,y)}{(\partial x)^\kappa} \right| \le  \sum_{\kappa: |\kappa| = \ell} b a^{\ell}
=b a^{\ell} \binom{\ell+d-1}{d-1} 
$$
Define an $n\times n$ matrix $\mtx{N}$ such that 
$
N_{ij} = P_{\calE_k, \ell} (\vct{x}_i; \vct{x}_j).
$
Then for all $1 \le i \neq j \le n$, 
$$
| M_{ij} - N_{ij} | = 
| f( \vct{x}_i, \vct{x}_j) -  P_{\calE_k, \ell-1} 
(\vct{x}_i; \vct{x}_j) |
\le  b a^{\ell}  \binom{\ell+d-1}{d-1}  k^{-\ell}.
 $$
 Moreover, for $ 1 \le i \le n,$ since $M_{ii}=0$, we get that 
$$
| M_{ii} - N_{ii} | =  |N_{ii}| = \left| P_{\calE_k, \ell-1} 
(\vct{x}_i; \vct{x}_i) \right|
\le |f(x_i, x_i) | + b a^{\ell} (\ell+d)^{d}  k^{-\ell}
\le 1 + b a^{\ell} \binom{\ell+d-1}{d-1} 
 k^{-\ell}.
$$
In the proof of \prettyref{prop:decay_rate}, we have already shown that 
the rank of $\mtx{N}$ is at most $k^d C_0(\ell,d)$ where
$C_0(\ell,d)=\sum_{i=0}^{\ell-1} \binom{i+d-1}{d-1}.$

We set $k=\lceil ea \rceil$, \ie, the smallest integer strictly larger than
$ea$. Define
$$
 r_0= \min \left\{ r \ge \lceil ea \rceil^d :  r^{1/d} \ge 2  \lceil ea \rceil  \log r \right\}.
 $$ 
 For any natural $r$,  if $ r \le r_0$, then by choosing $c_1 \ge \exp( r_0^{1/d} )$, we
 have that 
 $$
 \frac{1}{n^2} \sum_{i \ge r+1} \lambda_i^2(M) \le \frac{1}{n^2} \|M\|_F^2 \le 1 \le 
 c_1 \exp \left( -r^{1/d}  \right).
 $$ 
Next, we focus on the case of $r \ge r_0$. Then there exists an integer $\ell \ge 1$ such that
$k^d C_0(\ell, d) \le r \le k^d C_0(\ell+1, d)$. 
Note that 
\begin{align}
\binom{\ell+d-1}{d-1}   =  \frac{\ell+d-1}{\ell}  \binom{\ell+d-2}{d-1}\le d \binom{\ell+d-2}{d-1}
\le d C_0(\ell,d). \label{eq:binomial_bound}
\end{align}
It follows that 
\begin{align*}
 \frac{1}{n^2} \sum_{i\ge r+1}^n \lambda_i ( \mtx{M} )^2 & \le  
  \frac{1}{n^2} \sum_{i \ge k^d C_0(\ell,d)+1}^n \lambda_i ( \mtx{M} )^2  \le \frac{1}{n^2} \|\mtx{M}- \mtx{N} \|_F^2  \\
& \overset{(a)}{\le}   \frac{2}{n} + 2 b^2  a^{2\ell} d^2 C^2_0(\ell, d) k^{-2\ell}  \\
& \overset{(b)}{\le} \frac{2}{n} + 2 b^2  d^2   a^{-2d} r^2 e^{-2 (\ell +d) } \\  
& \overset{(c)}{\le} \frac{2}{n} + 2 b^2  d^2  a^{-2d} r^2  \exp \left(- \frac{2}{\lceil ea \rceil} r^{1/d} \right) \\
& \le \frac{2}{n} + 2 b^2 d^2 a^{-2d} \exp \left(- \frac{1}{\lceil ea \rceil} r^{1/d} \right).
\end{align*}
where in $(a)$ we used \prettyref{eq:binomial_bound};
$(b)$ follows due to $ r \ge k^d C_0 (\ell, d) $ and $ k \ge ea$;
$(c)$ holds because $r \le k^d C_0(\ell+1, d)$ and 
$$
C_0(\ell+1, d) =\sum_{i=0}^{\ell} \binom{i+d-1}{d-1}
\le (\ell+1) (\ell+d-1)^{d-1} \le (\ell+d)^d;
$$
the last inequality holds because $ r \ge r_0$. 
%
%
 Hence, the eigenvalues of $\mtx{M}$ has a super-polynomial decay with
 rate $\alpha=1/d$.  The theorem then follows by applying \prettyref{cor:main}.

\subsection{Proof of \prettyref{thm:operator}}
For a given integer $r \ge 0$, define a matrix $\mtx{N} \in \reals^{n\times n}$ with
$
N_{ij} = \sum_{k=1}^r \lambda_k (\calT)  \phi_k(x_i) \phi_k(x_j).
$
Note that when $r=0$, we set $N$ to be zero matrix. 
Then $\mtx{N}$ is of rank at most $r$. 
Therefore, $\sum_{k \ge r+1}  \lambda_k^2 (\mtx{M}) 
\le  \| \mtx{M} - \mtx{N} \|_F^2$ 
and thus to prove the theorem, 
it suffices to upper bound 
$\expect{\| \mtx{M} - \mtx{N} \|_F^2}$.

Indeed, because $M_{ii}=0$ and 
$M_{i,j}$ are identically distributed for 
$i \neq j$, we have that
$$
\expect{\| \mtx{M} - \mtx{N} \|_F^2} 
=n
\expect{ \left( \sum_{k=1}^r \lambda_k (\calT) \phi_k^2 (x_1)  \right)^2}
+ n(n-1)
\expect{ \left( M_{12} -  \sum_{k=1}^r \lambda_k (\calT)  \phi_k(x_1) \phi_k(x_2) \right)^2 }.
$$
For the first term in the last displayed equation, note that 
$$
\expect{ \left( \sum_{k=1}^r \lambda_k (\calT) \phi_k^2 (x_1)  \right)^2}
=  \sum_{k=1}^r \sum_{\ell=1}^r \lambda_k(\calT) \lambda_\ell(\calT) 
\expect{\phi_k^2 (x_1) \phi_\ell^2(x_1) }.
$$
For the second term, note that
$$
\expect{ \left( M_{12} -  \sum_{k=1}^r \lambda_k (\calT)  \phi_k(x_1) \phi_k(x_2) \right)^2 } =\left\|  f(x_1, x_2) -   \sum_{k=1}^r \lambda_k (\calT)  \phi_k(x_1) \phi_k(x_2)  \right\|^2_2,
$$
where the $2$-norm denotes the $L^2(\calX \times \calX, \mu \otimes \mu)$ norm. 
For any  integer $m \ge r$, by Minkowski's inequality,
\begin{align*}
& \left\|  f(x_1, x_2) -   \sum_{k=1}^r \lambda_k (\calT)  \phi_k(x_1) \phi_k(x_2)  \right\|_2 \\
& \le  \left\|  f(x_1, x_2) -   \sum_{k=1}^m \lambda_k (\calT)  \phi_k(x_1) \phi_k(x_2)  \right\|_2
+  \left\| \sum_{k=r+1}^{m} \lambda_k (\calT)  \phi_k(x_1) \phi_k(x_2) \right\|_2 \\
& =  \left\|  f(x_1, x_2) -   \sum_{k=1}^m \lambda_k (\calT)  \phi_k(x_1) \phi_k(x_2)  \right\|_2 + \sqrt{\sum_{k=r+1}^m \lambda_k^2(\calT)} ,
\end{align*}
where the last inequality follows because $\expect{\phi_k(x_i) \phi_\ell (x_i) }=\delta_{k\ell}$ and $x_i$'s are independent. 
In view of \prettyref{eq:kernalapprox} and the fact that $\|f(x_1,x_2) \|_2$ is bounded, we get that $\sum_{k=r+1}^\infty \lambda_k^2(\calT)$ exists and is bounded.
By taking the square and then letting $m \to \infty$ in both hand sides of the last displayed equation,
we get  that
$$
\left\|  f(x_1, x_2) -   \sum_{k=1}^r \lambda_k (\calT)  \phi_k(x_1) \phi_k(x_2)  \right \|_2^2 \le  \sum_{k=r+1}^\infty \lambda_k^2(\calT),
$$
Therefore, 
$$
 \expect{ 
\left( M_{12} -  \sum_{k=1}^r \lambda_k (\calT)  \phi_k(x_1) \phi_k(x_2) \right)^2 }
\le  \sum_{k=r+1}^\infty \lambda_k^2(\calT),
$$
which completes the proof.

%
%

	\section{Numerical examples}\label{sec:numerical}
In this section, we provide numerical results on
synthetic datasets, which corroborate 
our theoretical results. We assume the sparsity level $\rho$ is known and 
set the threshold $\tau=2.01\sqrt{n \rho}$ throughout
the experiments. In the case where $\rho$ is unknown,
one can apply cross-validation procedure to adaptively
choose the sparsity level $\rho$ as shown in~\cite{gao2016optimal}.
We first apply USVT with input $(A, \tau, \rho)$,
and then 
output the estimator $\hat{M}$, and finally 
calculate the MSE error $\MSE(\hat{M})$.

\subsection{Stochastic block model}
For a fixed  number of blocks $k$, we randomly generate a $k\times k $
symmetric matrix $B$ such that for $i \le j$, 
$B_{ij}=B_{ji}$ are independently and uniformly generated from $[0,1]$. 
For a fixed integer $n$ which divides $k$, we partition the vertex set $[n]$ into 
$k$ communities of equal sizes uniformly at random. 
Given $B$, a community partition $\{S_\ell\}_{\ell=1}^k$,
and observation probability $\rho$, 
an adjacency matrix $A$ is generated with the edge probability between
node $i \in S_\ell$ and node $ j \in S_{\ell'}$ being
$\rho \times M_{ij}$, where $M_{ij}=B_{\ell\ell'}$.

We first simulate SBM with a fixed sparsity level $\rho=0.1$ 
and a varying number of  blocks $k \in \{2, 4, 8, 16\}$. The simulation
results are depicted in Fig.~\ref{fig:usvt_sbm}. 
Panel (a) shows the MSE of the USVT decreases 
as the number of vertices $n$ increases.  
Our theoretical result suggests that the rate of convergence
of MSE is $\frac{k}{n\rho} \wedge 1$. 
In Panel (b),
we rescale the $x$-axis to $\log (n \rho/k)$, 
and the $y$-axis to the log of MSE. 
The curves for different $k$ align well with
each other and decreases linearly with a slope of
approximately $1$, as predicted by our theory. 
We next simulate SBM with a fixed number of blocks $k=4$ 
and a varying sparsity level $\rho \in \{0.4, 0.2, 0.1, 0.05\}$. 
The results are depicted in Fig.~\ref{fig:usvt_sbm_varyingrho}. Again 
after rescaling, 
the curves for different observation probabilities $\rho$ 
align well with each other and decrease linearly with a rate of 
approximately $1$.

\begin{figure}[H]
\centering
\begin{tabular}{cc}
\includegraphics[width=.45\columnwidth]{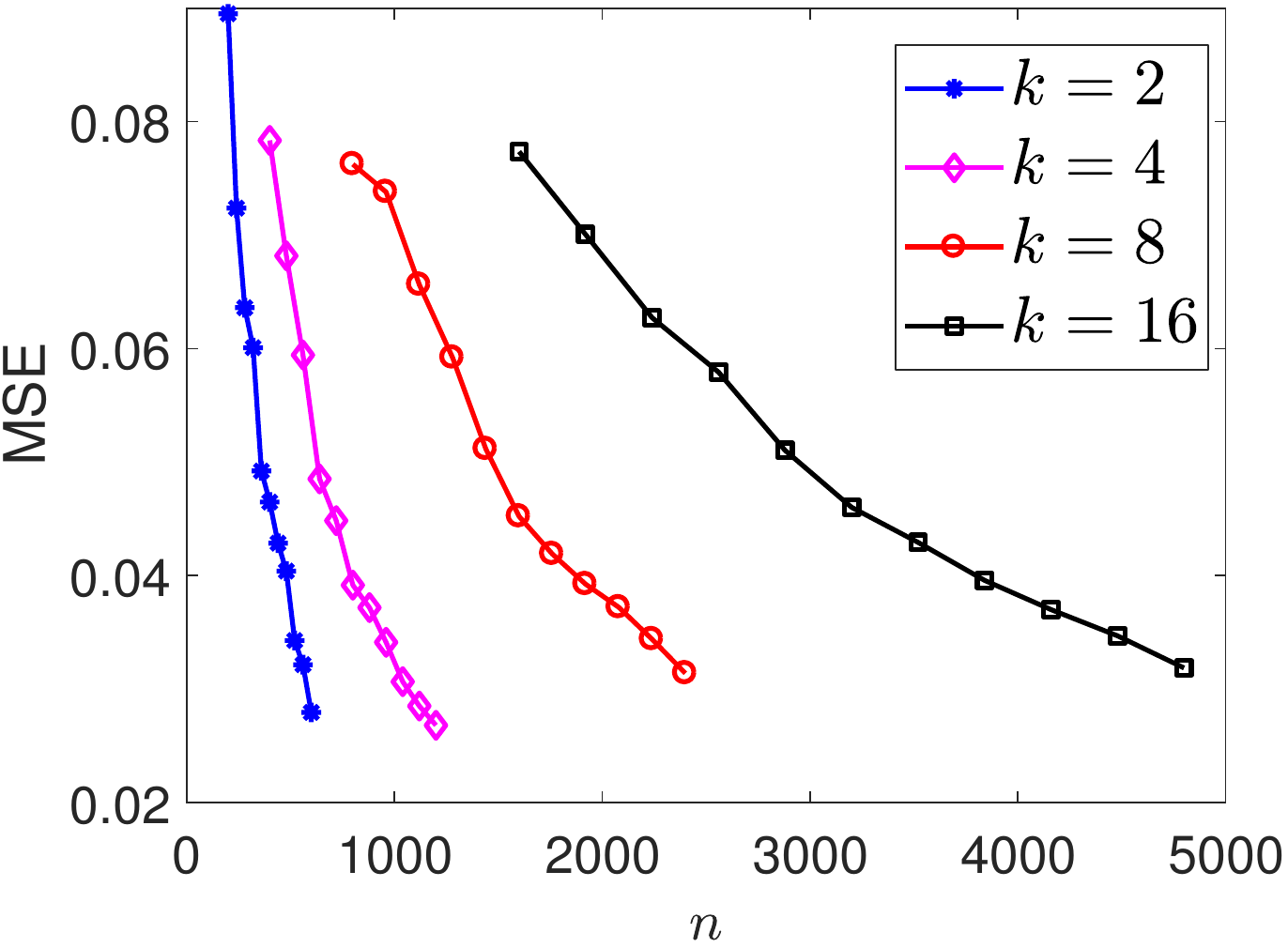}
&  \includegraphics[width=.45\columnwidth]{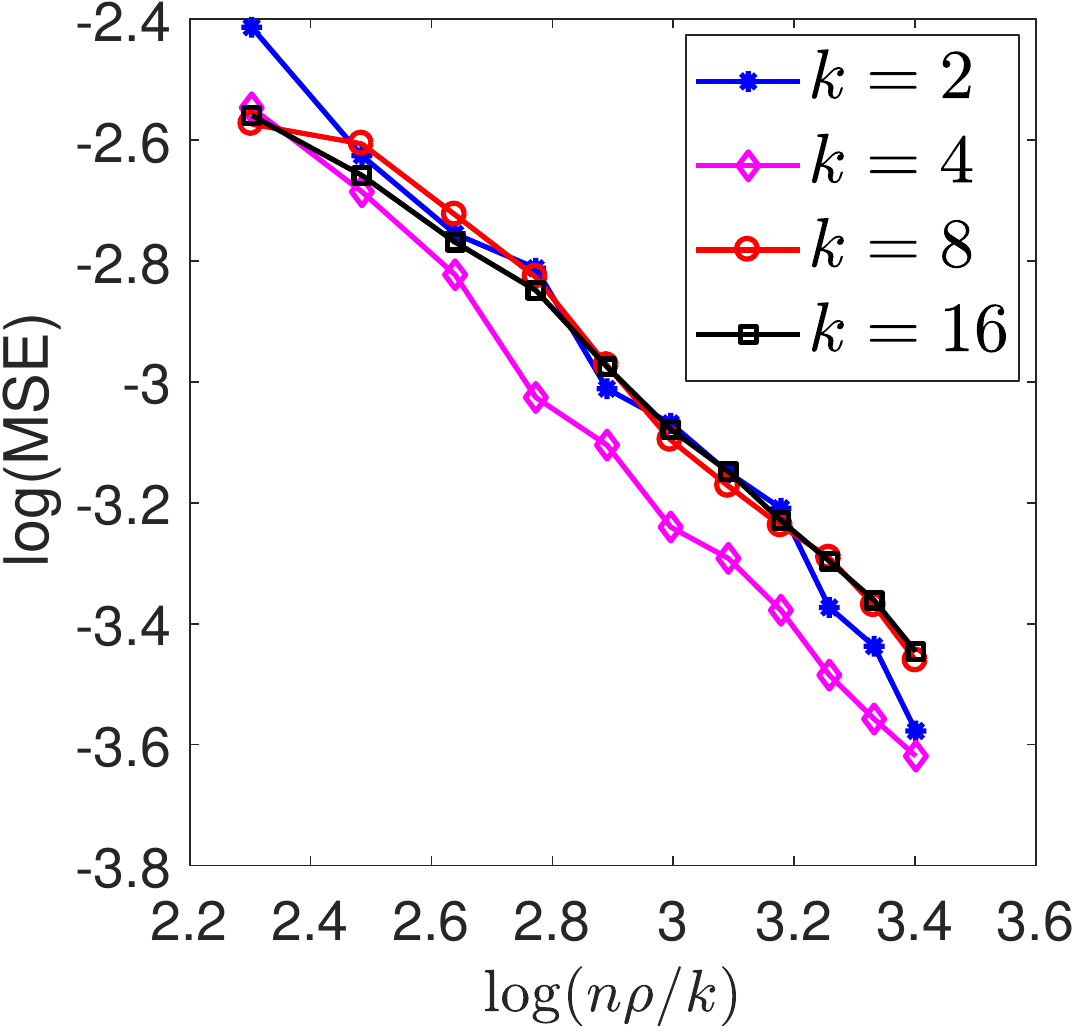} 
\\
(a)	& (b)  
\end{tabular}
\caption{The MSE error of USVT estimator under stochastic block models 
for varying number of blocks 
$k$ and a fixed observation probability $\rho=0.1$. 
Panel (a): MSE versus the number of vertices $n$; 
Panel (b): The log of MSE versus $\log (n \rho /k)$.
Each point
represents the average of MSE over $20$ independent runs.}
\label{fig:usvt_sbm}
\end{figure}

\begin{figure}[H]
\centering
\begin{tabular}{cc}
\includegraphics[width=.45\columnwidth]{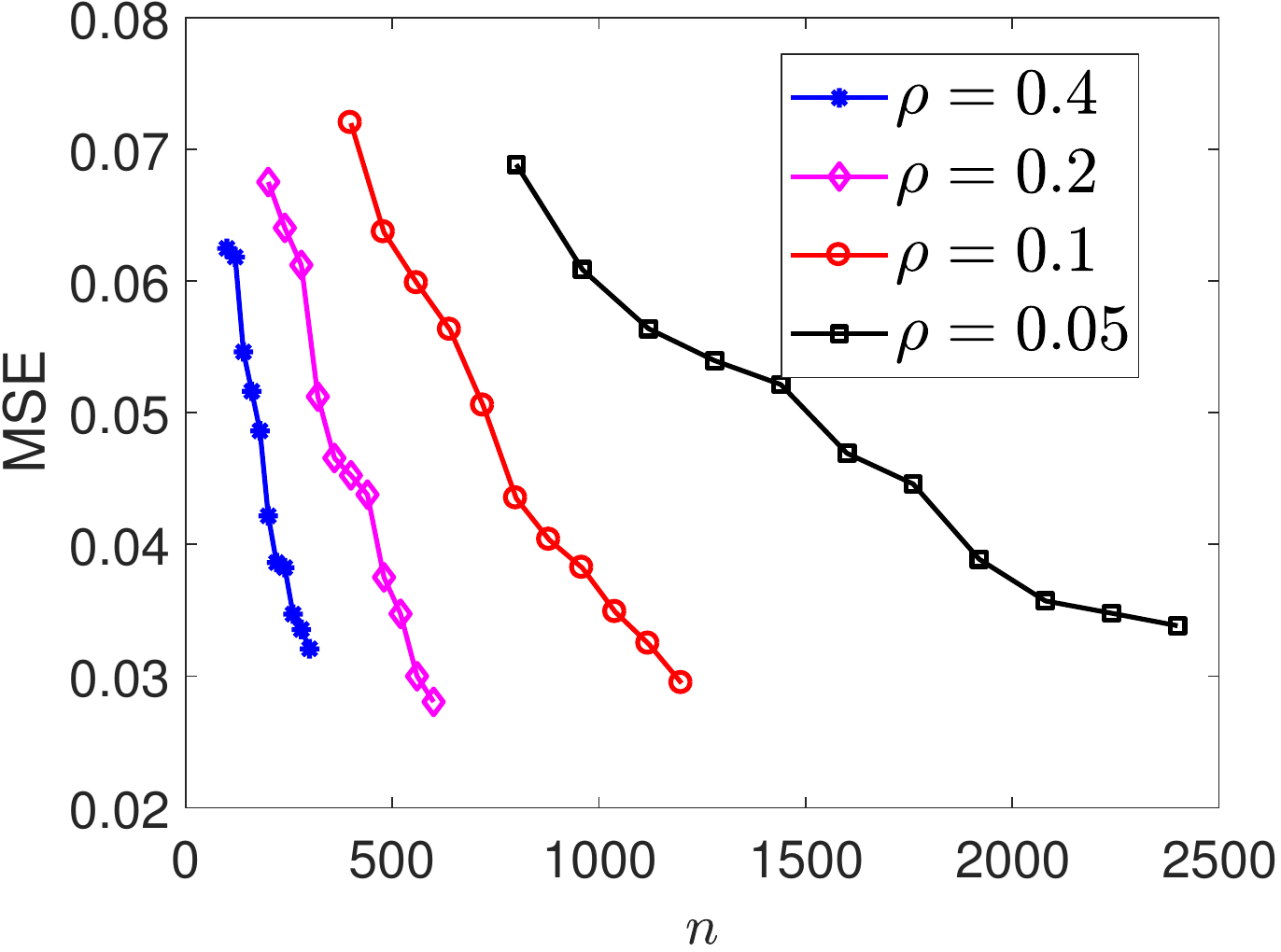}
&  \includegraphics[width=.45\columnwidth]{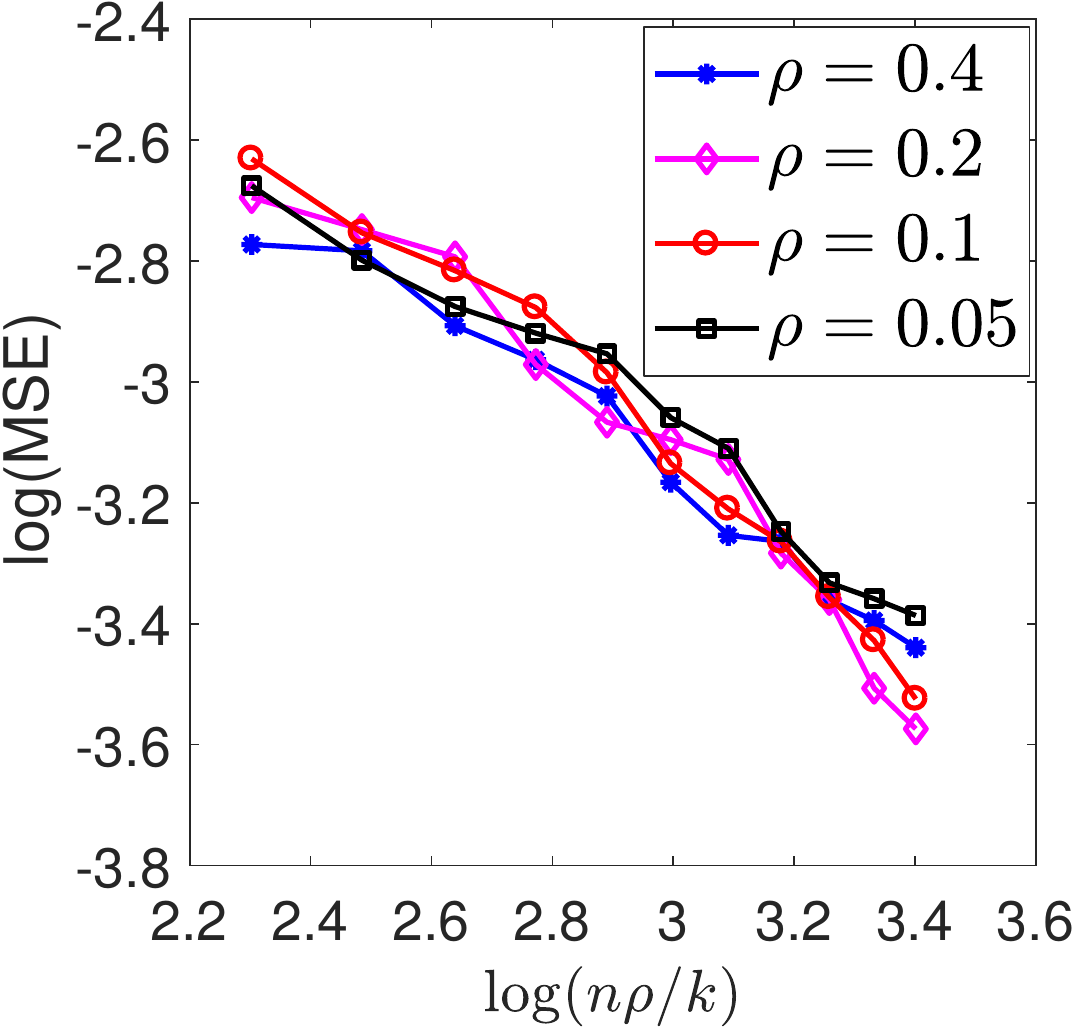} 
\\
(a)	& (b)  
\end{tabular}
\caption{The MSE error of USVT estimator under stochastic block models 
for varying observation probabilities
 and a fixed number of blocks $k=4$. 
Panel (a): MSE versus the number of vertices $n$; 
Panel (b): The log of MSE versus $\log (n \rho /k)$.
Each point
represents the average of MSE over $30$ independent runs. }
\label{fig:usvt_sbm_varyingrho}
\end{figure}

\subsection{Translation invariant graphon}
 For some $a>0$, let $h: [-a,a] \to \reals$ 
 denote an even  function, \ie, $h(x)=h(-x)$.
 Let us extends its domain to the real line by the periodic 
 extension such that $h(x+ 2k a) = h(x)$ for all
 $x \in [-a,a]$ and integers $k \in \integers.$
 By construction $h$ has a period $2a$.
 Using this function, we can define a translation-invariant
 graphon on the product space $[-a,a]\times [-a,a]$
 via $f(x,y)=h(x-y)$. Since $h$ is even, it follows that
 $f$ is symmetric. 
 Then the integral operator $\calT$ defined 
 in \prettyref{eq:operator} reduces to
 $$
 \left( \calT g \right) (x)= \frac{1}{2a} \int_{-a}^{a} h(x-y) g(y) \diff y
 = \frac{1}{2a} \left(h \ast g \right) (x), \quad \forall x \in [-a, a], 
 $$
 where $\ast$ denotes the convolution. 
 Hence, we can explicitly determine
 the eigenvalues of $\calT$ via Fourier analysis.
 In particular, suppose that $h$ has the following Fourier series expansion:
 $$
 h (x) = \sum_{k=-\infty}^\infty \hat{h}[k] e^{j \pi k x /a}, \quad
 \hat{h} [k] = \frac{1}{2a} \int_{-a}^{a} h (x) e^{-j \pi k x /a}
 \diff x.
 $$
 where throughout this section $j$ denotes the imaginary part such that $j^2=-1$,
 and $\hat{h}[k]$ are the Fourier coefficients. 
 Since $h$ is even, it follows that $\hat{h}[k]$'s are real 
 and $\hat{h}[k]=\hat{h}[-k]$. Fourier analysis entails 
 a one-to-one correspondence 
 between eigenvalues of $\calT$ and Fourier coefficients of $h$:
  $\lambda_k(\calT) = \hat{h}[k]$.

We specify $h: [-1, 1] \to \reals$ as $h(x) =|x|$ and 
simulate the graphon model with $f(x,y) = h(x-y)$ 
for $x, y \in [-1, 1]$ and the underlying measure $\mu$
being uniform over $[-1,1]$.  
Since $h(x) =|x|$, the Fourier coefficients
can be explicitly computed as 
$\lambda_k (\calT) =\hat{h}[k] = 2 \sin^2(\pi k /2)/(\pi^2 k^2 )$
with eigenfunctions
given by 
$\{ \cos(\pi k x)\}_{k=0}^\infty$ and $ \sin(\pi kx)\}_{k=1}^\infty$.
It follows from \prettyref{thm:operator} that the eigenvalues of 
$M$ satisfy
$$
\frac{1}{n^2} \sum_{i \ge r+1} \expect{ \lambda_i^2(M) }
\le O(n^{-1})  + O(r^{-3})
$$
uniformly over all integers $r \ge 0.$ 
Therefore, our theory predicts that  the MSE of USVT converges to 
zero at least in a rate of $(n\rho)^{-3/4}$. 
The simulation results for varying observation probabilities 
are depicted in Fig.~\ref{fig:usvt_transinv}. 
Panel (a) shows the MSE  converges to $0$ 
as the number of vertices $n$ increases. 
In Panel (b), we rescale the $x$-axis to $\log (n\rho) $ 
and the $y$-axis to the log of MSE. The curves for different $\rho$ align well with each other after the rescaling and decrease linearly with a slope of approximately $0.8$, which is close to $3/4$ as predicted by our theory.

\begin{figure}[H]
\centering
\begin{tabular}{cc}
\includegraphics[width=.45\columnwidth]{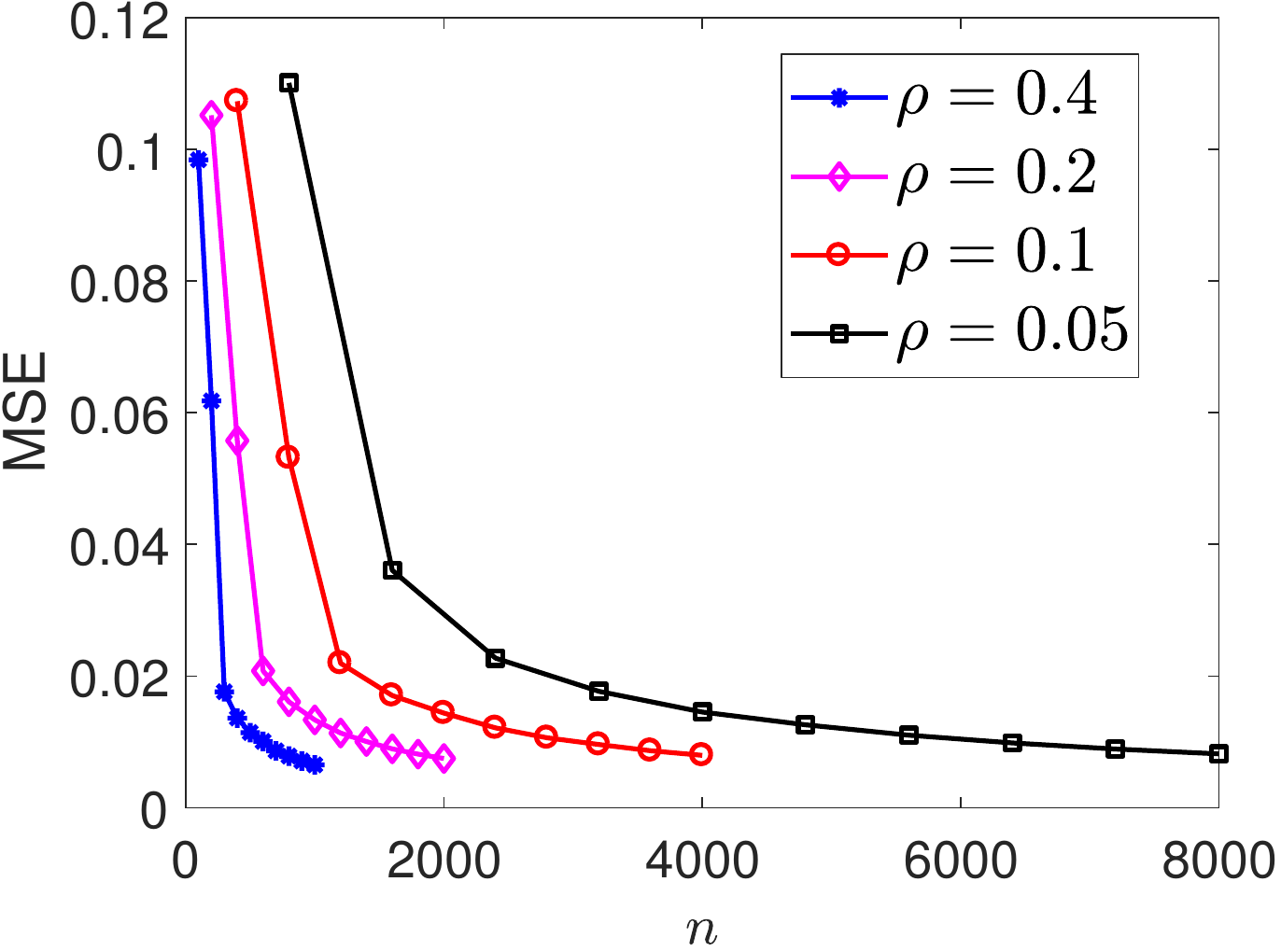}
&  \includegraphics[width=.45\columnwidth]{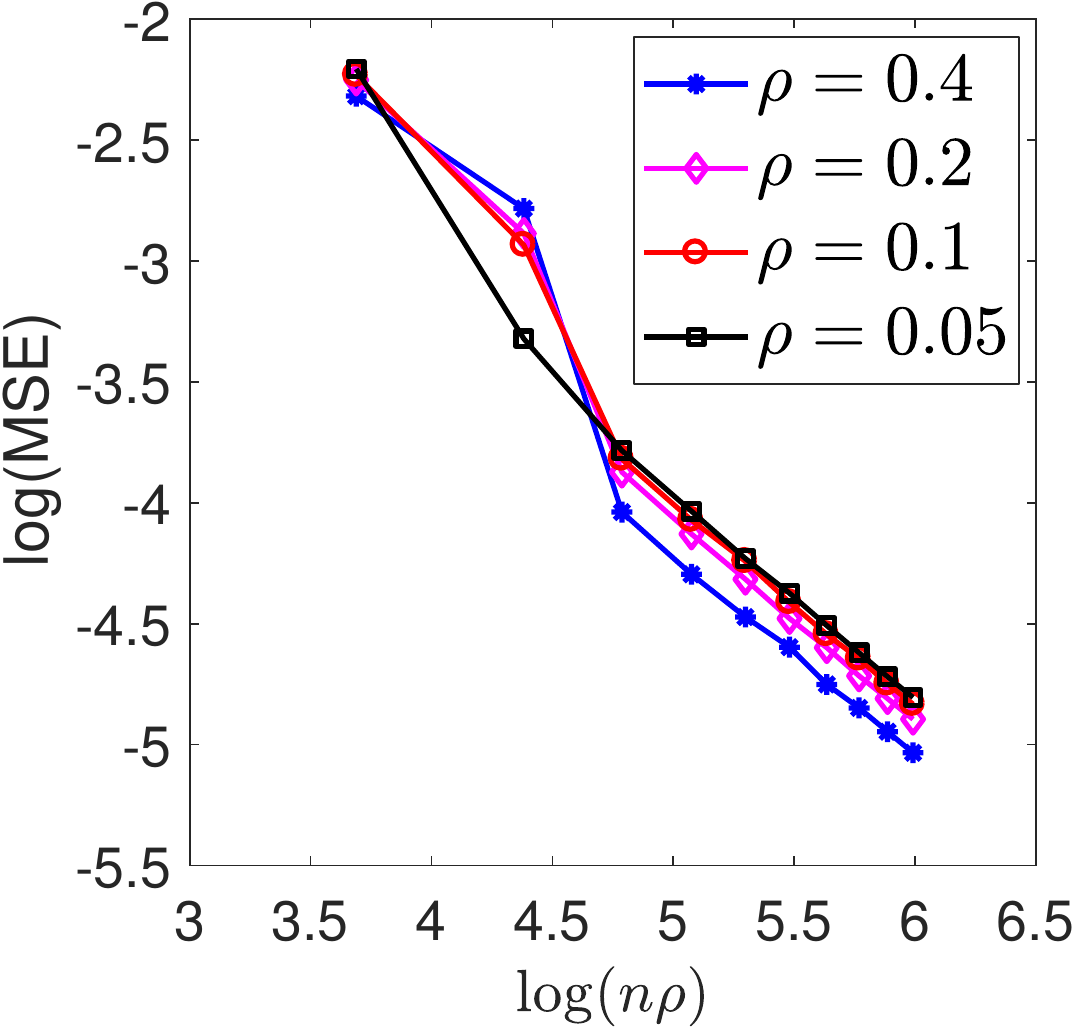} 
\\
(a)	& (b)  
\end{tabular}
\caption{The MSE error of USVT estimator under the translation invariant graphon
$f(x,y)=|x-y|$. 
Panel (a): MSE versus the number of vertices $n$; 
Panel (b): The log of MSE versus $\log( n \rho) .$ 
Each point
represents the average of MSE over $10$ independent runs.}
\label{fig:usvt_transinv}
\end{figure}

\subsection{Sobolev graphon}
In this section, we simulate the graphon
model with $\calX=[0,1]$ and
 $\mu$ being the uniform measure and 
 $f(x,y) = \min\{x, y\}$. 
 Then $\nabla_x f(x,y) = \indc{x \le y}$
 and $\nabla_y f(x,y) = \indc{ y \le x}$.
 Moreover, $| f(x,y) - f(x',y') | \le |x-x'| + |y-y'|$.
 However, the second-order weak derivatives 
 of $f$ do not exist. 
 Therefore, $f$ is Sobolev smooth with $\alpha=1.$
 In this case, one can get a bound  on the eigenvalue decay rate tighter 
 than \prettyref{prop:decay_rate} 
 by directly computing $\lambda_n(\calT)$ and invoking \prettyref{thm:operator}.
 Note that 
 $$
 \left( \calT g \right) (x) = \int_{0}^1 \min\{x, y \} g(y) \diff y
 = \int_{0}^x y g(y) \diff y + x \int_{x}^1 g(y) \diff y.
 $$
 Suppose $\phi$ is an eigenfunction of $\calT$ with
 eigenvalue $\lambda$. Then
 $$
 \int_{0}^x y \phi (y) \diff y + x \int_{x}^1 \phi(y) \diff y = \lambda \phi(x).
 $$
 It follows that $\phi(0)=1$ and $\lambda \phi'(x) = \int_{x}^1 \phi(y) \diff y$.
 It further implies that $\phi'(1)=0$ and $\lambda \phi'' + \phi =0.$
 Therefore, the eigenfunction and eigenvalue pairs are given by
 $$
 \phi_k (x) = \sin \frac{ (2k-1) \pi x}{2}, \text{ and } \lambda_k(\calT) = \left( \frac{2}{(2k-1) \pi }\right)^2.
 $$
It follows from \prettyref{thm:operator} that the eigenvalues of 
$M$ satisfy 
$$
\frac{1}{n^2} \sum_{i \ge r+1} \expect{ \lambda_i^2(M) }
\le O(n^{-1})  + O(r^{-3})
$$ 
uniformly over all integers $r \ge 0$. 
Therefore, our theory predicts that  the MSE of USVT converges to 
zero in a rate of $(n\rho)^{-3/4}$. 
The simulation results for varying observation probabilities 
are depicted in Fig.~\ref{fig:usvt_transinv}. 
The curves in Panel (b) for different $\rho$ align well with each other after the rescaling
and decrease linearly with a slope of approximately $0.7$, which is close to
$3/4$ as predicted by our theory.

\begin{figure}[H]
\centering
\begin{tabular}{cc}
\includegraphics[width=.45\columnwidth]{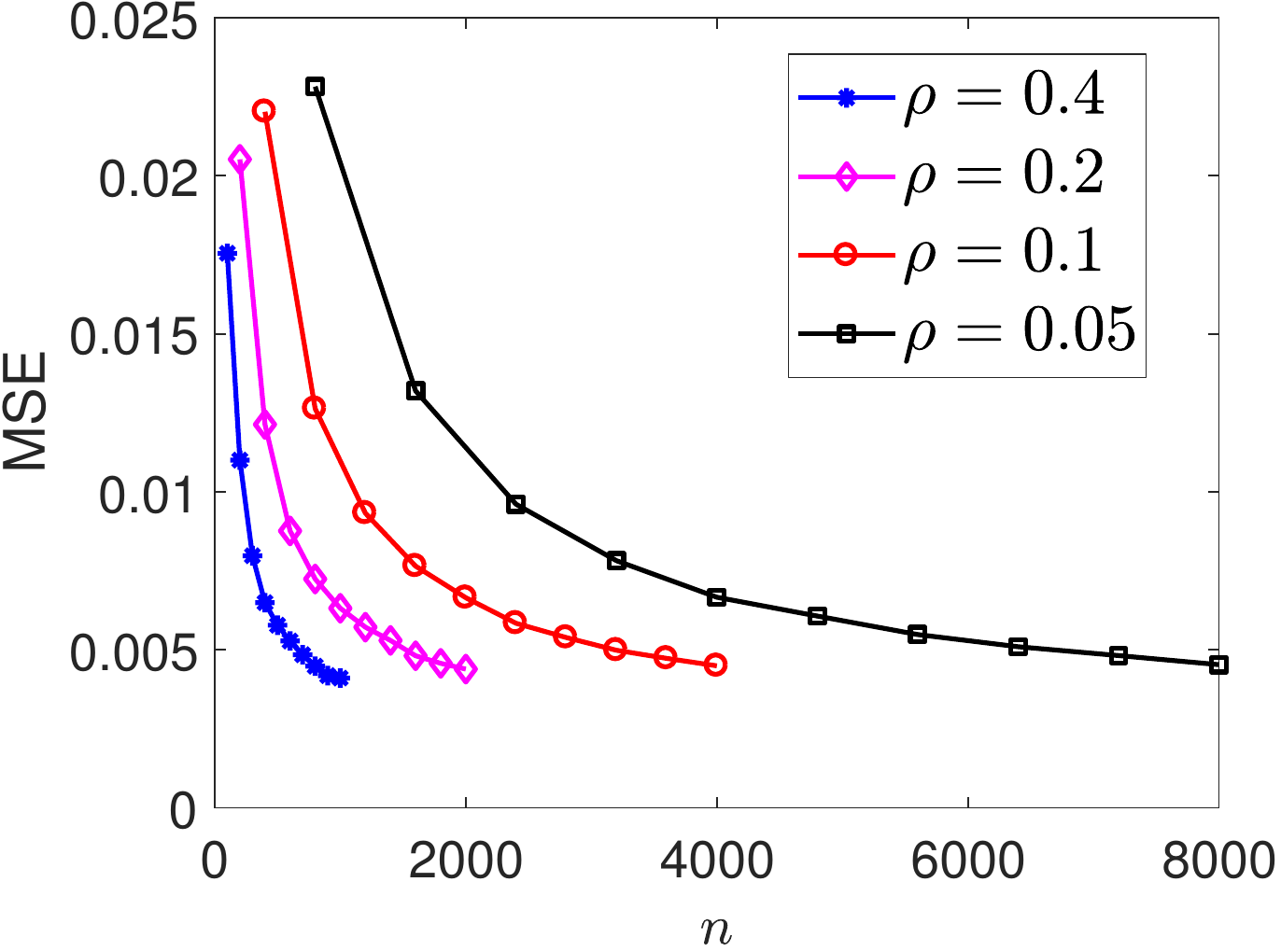}
&  \includegraphics[width=.45\columnwidth]{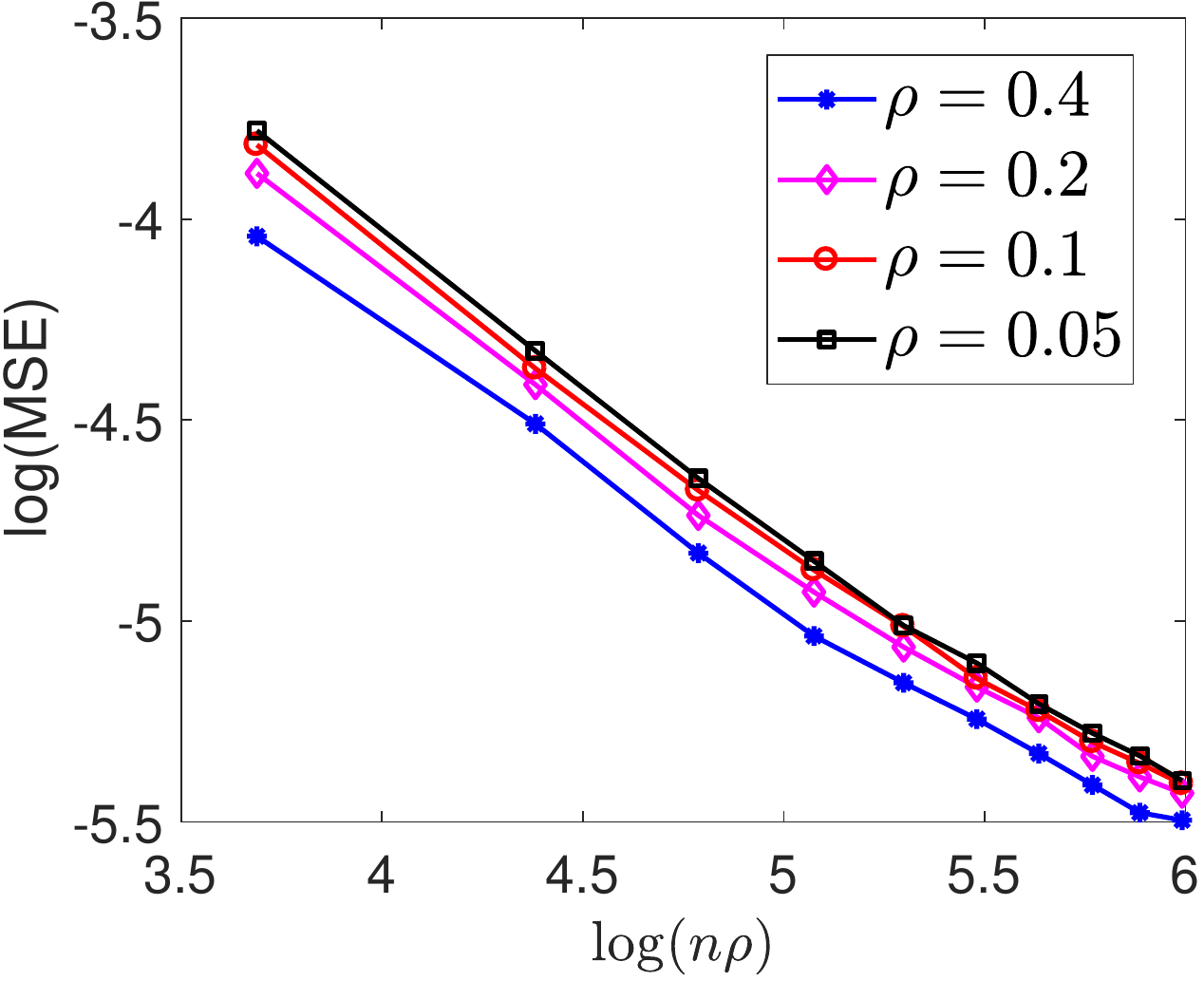} 
\\
(a)	& (b)  
\end{tabular}
\caption{The MSE error of USVT estimator under the first-order sobolev graphon
$f(x,y)=\min\{x, y\}.$ 
Panel (a): MSE versus the number of vertices $n$; 
Panel (b): The log of MSE versus $\log( n \rho) .$ 
Each point
represents the average of MSE over $10$ independent runs.}
\label{fig:usvt_sobolev}
\end{figure}

\section{Conclusions and future work}
In this paper, we establish upper bounds to the graphon estimation
error of the universal singular value thresholding algorithm in the
relatively sparse regime where the average  vertex degree is at least 
logarithmic in $n.$ In both the stochastic block model setting 
and the smooth graphon setting, we show that the estimation error
of USVT converges to $0$ as $n \to \infty$. Moreover, when
 graphon function $f$ belongs to \Hold or Sobolev space with smootheness
 index $\alpha$, we show that the rate of convergence is  at most
 $(n\rho)^{-2\alpha/(2\alpha+d)}$,
 approaching the minimax optimal rate $\log(n\rho)/(n\rho)$ proved in the literature for $d=1$, as 
 $f$ gest smoother. Furthermore, when $f$ is analytic with infintely many
 times differentiability, we show the rate of convergence is at most $\log^d(n\rho)/(n\rho)$.

A future direction important in both theory and practice is to develop computationally efficient graphon estimation procedures  in networks with bounded average degrees  and characterize the rate of convergence of the estimation error.
Another fundamental and open question is whether the minimax optimal rate can be achieved in polynomial-time. For stochastic block models with $k$ blocks, we observe a multiplicative gap of $k/\log k$ between the rate of convergence of USVT and the minimax optimal rate. For \Hold or Sobolev smooth graphons with smoothness index $\alpha$
and the latent feature space of dimension $d=1$,
we observe a multiplicative gap of $(n\rho)^{1/(2\alpha+1)}/\log(n\rho)$ between the rate of convergence of USVT and the minimax optimal rate. The minimax optimal rates are
unknown for \Hold or Sobolev smooth graphons with $d>1$ and analytic graphons with $d\ge 1.$

\section*{Acknowledgement}
The author would like to thank Yudong Chen, Christina Lee, and Yihong Wu for inspiring discussions
on spectral methods for graphon estimation.

\bibliographystyle{abbrv}

 \bibliography{graphon,graphical_combined}
 
 \begin{appendix}
 \section{Proof of \prettyref{eq:optimal_rate}} \label{app:optimal_rate}
 It has been shown in~\cite{klopp2015oracle,gao2016optimal} that the minimax
 optimal error rate of estimating $\alpha$-\Hold smooth graphon is given by:
 $$
\inf_{\hat{\mtx{M}}} \sup_{f \in \calH(\alpha, L) } \sup_{\mu \in \calP[0,1]} 
\MSE(\hat{\mtx{M}}) 
\asymp  \min_{ 1 \le k \le n} \left\{ \frac{k^2}{n^2 \rho} + \frac{\log k}{n \rho}
+  k^{-2 (\alpha \wedge 1)} \right\} \wedge 1.
$$ 
Next, we solve the above minimization problem over $k$ by dividing the analysis into four cases.
Combining all four cases completes the proof.  

{\bf Case 1:} $\log (n\rho) \ge \alpha \log n + (\alpha+1)
\log \log n $.
In this case, we must have $\alpha \le 1$.  We set 
$k= \lfloor (n^2\rho)^{1/(2\alpha+2)} \rfloor $
and get that 
\begin{align*}
\min_{ 1 \le k \le n} \left\{ \frac{k^2}{n^2 \rho} + \frac{\log k}{n \rho}
+  k^{-2 (\alpha \wedge 1)} \right\}
& \le 2 (n^2 \rho)^{ -\alpha / (\alpha+1) }
+ \frac{1}{2\alpha+2} \frac{\log(n^2 \rho)}{n \rho} \\
& \le  2 (n^2 \rho)^{ -\alpha / (\alpha+1) }
+ \frac{\log n}{n \rho} \\
& \le 3 (n^2 \rho)^{ -\alpha / (\alpha+1) }.
\end{align*}
where the last inequality holds because
$\log (n\rho) \ge \alpha \log n + (\alpha+1)
\log \log n $ is equivalent to
$
(n^2 \rho)^{ -\alpha / (\alpha+1) } \ge \log n /(n\rho).$

On the contrary, 
$$
\min_{ 1 \le k \le n} \left\{ \frac{k^2}{n^2 \rho} + \frac{\log k}{n \rho}
+  k^{-2 (\alpha \wedge 1)} \right\}
\ge \min_{ 1 \le k \le n} \left\{ \frac{k^2}{n^2 \rho} +  k^{-2\alpha } \right\}
\ge (n^2\rho)^{-\alpha/(\alpha+1)}.
$$


{\bf Case 2:} $\alpha \log n \le \log (n\rho) \le \alpha \log n + (\alpha+1)
\log \log n $. In this case, we still have $\alpha \le 1$ and set 
$k= \lfloor (n^2\rho)^{1/(2\alpha+2)} \rfloor $. We get that 
\begin{align*}
\min_{ 1 \le k \le n} \left\{ \frac{k^2}{n^2 \rho} + \frac{\log k}{n \rho}
+  k^{-2 (\alpha \wedge 1)} \right\}
& \le  2 (n^2 \rho)^{ -\alpha / (\alpha+1) }
+ \frac{\log n}{n \rho} \\
& \le  \frac{3\log n}{n \rho}
\le \frac{3 \log (n\rho)}{\alpha n\rho},
\end{align*}
where in the last two inequalities we used the 
assumption that $\alpha \log n \le \log (n\rho) \le \alpha \log n + (\alpha+1)
\log \log n $.

On the contrary,
$$
\min_{ 1 \le k \le n} \left\{ \frac{k^2}{n^2 \rho} + \frac{\log k}{n \rho}
+  k^{-2 (\alpha \wedge 1)} \right\}
\ge \min_{ 1 \le k \le n} \left\{ \frac{\log k}{n \rho} +  k^{-2\alpha } \right\}
\ge \frac{ \log (n\rho)}{4\alpha n\rho}.
$$

{\bf Case 3:} $\omega(1)=\log (n\rho) \le \alpha \log n$. 
In this case, we set
$$
k=\lfloor (n\rho)^{ \frac{1}{2 (\alpha \wedge 1)} } \rfloor 
$$ 
and get that 
$$
\min_{ 1 \le k \le n} \left\{ \frac{k^2}{n^2 \rho} + \frac{\log k}{n \rho}
+  k^{-2 (\alpha \wedge 1)} \right\}
\le  \frac{ \left( n \rho \right)^{  \frac{1}{\alpha \wedge 1}  }  }{n^2 \rho} 
+ \frac{1}{2 (\alpha \wedge 1) } \frac{\log (n\rho)}{n \rho} + \frac{1}{n \rho}
\le \frac{2}{n \rho} + 
\frac{1}{2 (\alpha \wedge 1) } \frac{\log (n\rho)}{n \rho},
$$
where the last inequality holds because 
$(n\rho)^{1/(\alpha\wedge 1)} \le n$. The proof of 
the lower bound is similar to  that in Case 2. 

{\bf Case 4:} $n \rho = O(1)$. In this case, we trivially have
$$
\min_{ 1 \le k \le n} \left\{ \frac{k^2}{n^2 \rho} + \frac{\log k}{n \rho}
+  k^{-2 (\alpha \wedge 1)} \right\} \wedge 1 \asymp 1.
$$

 \end{appendix}

\end{document}